\documentclass[11pt]{article}
\usepackage{dsfont}
\usepackage{environ}

\usepackage[utf8x]{inputenc}

\usepackage[usenames,dvipsnames]{xcolor}
\definecolor{Gred}{RGB}{219, 50, 54}
\definecolor{Ggreen}{RGB}{60, 186, 84}
\definecolor{Gblue}{RGB}{72, 133, 237}
\definecolor{Gyellow}{RGB}{247, 178, 16}
\definecolor{ToCgreen}{RGB}{0, 128, 0}
\definecolor{myGold}{RGB}{231,141,20}
\definecolor{myBlue}{rgb}{0.19,0.41,.65}
\definecolor{myPurple}{RGB}{175,0,124}
\definecolor{niceRed}{RGB}{153,0,0}
\definecolor{niceRed}{RGB}{190,38,38}
\definecolor{blueGrotto}{HTML}{059DC0}
\definecolor{royalBlue}{HTML}{057DCD}
\definecolor{navyBlueP}{HTML}{0B579C}
\definecolor{limeGreen}{HTML}{81B622}
\definecolor{ceruleanblue}{rgb}{0.16, 0.32, 0.75}

\usepackage{cmap}
\usepackage[T1]{fontenc}
\usepackage{bm}
\usepackage{amsmath, thmtools}
\usepackage{amsfonts}
\usepackage{amssymb}
\usepackage{amsbsy}
\usepackage{amsthm}
\usepackage{tcolorbox}

\usepackage{graphicx, ucs}
\usepackage{pgfplots}
\usepackage{caption}
\usepackage{subcaption}
\usepackage{rotating}
\usepackage{float}
\usepackage{tikz}

\usepackage{algorithm, caption}
\usepackage[noend]{algpseudocode}
\usepackage{listings}

\usepackage{enumitem}

\usepackage[pagebackref]{hyperref}
\hypersetup{
  colorlinks = true,
  urlcolor = {blue},
  linkcolor = {BlueViolet},
  citecolor = {ForestGreen}
}
\usepackage[nameinlink]{cleveref}

\usepackage{multirow}
\usepackage{array}

\usepackage{chngcntr}

\counterwithin{equation}{section}
\usepackage{chngcntr}
\usepackage{soul}
\usepackage{dsfont}
\usepackage{nicefrac}

\usepackage{fancyhdr}
\fancyheadoffset{0pt}
\def\compactify{\itemsep=0pt \topsep=0pt \partopsep=0pt \parsep=0pt}
\let\latexusecounter=\usecounter

\definecolor{myC}{rgb}{0, 255, 255}
\definecolor{myY}{rgb}{204, 204, 0}
\definecolor{myM}{rgb}{255, 0, 255}
\definecolor{secinhead}{RGB}{249,196,95}
\definecolor{lgray}{gray}{0.8}

\usepackage{helvet}

\usepackage{hyperref}
\usepackage{soul}

\usepackage{algorithm}
\usepackage{algpseudocode}

\newtheorem{theorem}{Theorem}  
\newtheorem{proposition}[theorem]{Proposition}
\newtheorem{corollary}[theorem]{Corollary}
\newtheorem{lemma}[theorem]{Lemma}
\newtheorem{definition}[theorem]{Definition}
\newtheorem{inftheorem}{Informal Theorem}
\newtheorem{fact}{Fact}
\newtheorem{assumption}{Assumption}
\newtheorem{claim}{Claim}
\newtheorem{remark}{Remark}
\newtheorem{step}{Step}
\newtheorem*{theorem*}{Theorem}

\newcommand{\reals}{\mathbb{R}}

\newcommand{\D}{\mathcal{D}}
\newcommand{\Q}{\mathcal{Q}}

\newcommand{\G}{\mathcal{G}}

\def\calE{\mathcal{E}}

\newcommand{\supp}{\mathrm{supp}}
\newcommand{\poly}{\mathrm{poly}}

\def\l{\ell}
\def\<{\langle}
\def\>{\rangle}
\def\wh{\widehat}

\def\E{ \mathbb{E}}

\DeclareMathOperator*{\argmax}{argmax}
\DeclareMathOperator*{\argmin}{argmin}

\def\wt{\widetilde}
\def\poly{\mathrm{poly}}

\def\vec{\bm}

\makeatletter
\renewenvironment{abstract}{%
    \if@twocolumn
      \section*{\abstractname}%
    \else 
      \begin{center}%
        {\bfseries \large\abstractname\vspace{\z@}}
      \end{center}%
      \quotation
    \fi}
    {\if@twocolumn\else\endquotation\fi}
\makeatother

\begin{document}

\title{
Perfect Sampling from Pairwise Comparisons
}
\author{
  \textbf{Dimitris Fotakis}\footnote{National Technical University of Athens, \url{fotakis@cs.ntua.gr}. Supported by the HFRI, project BALSAM, HFRI-FM17-1424.} \\
  \small NTUA 
  \and
  \textbf{Alkis Kalavasis}\footnote{National Technical University of Athens, \url{kalavasisalkis@mail.ntua.gr}. Supported by the HFRI, project BALSAM, HFRI-FM17-1424.} \\
 \small  NTUA 
  \and
  \textbf{Christos Tzamos}\footnote{University of Wisconsin-Madison \& University of Athens, \url{tzamos@wisc.edu}. Supported by NSF grants CCF-2144298 and CCF-2008006.
} \\
 \small UW Madison \& UOA
}
\maketitle
\thispagestyle{empty}

\begin{abstract}
\small
In this work, we study how to efficiently obtain perfect samples from a discrete distribution $\D$ given access
only to pairwise comparisons of elements of its support.
Specifically, we assume access to 
samples $(x, S)$, where $S$ is drawn from a distribution over sets $\Q$ (indicating the elements being compared), and $x$ is drawn from the conditional distribution $\D_S$ (indicating the winner of the comparison) and aim to output a clean sample $y$ distributed according to $\D$. We mainly focus on the case of pairwise comparisons where all sets $S$ have size 2. We design a Markov chain whose stationary distribution coincides with $\D$ and give an algorithm to obtain exact samples using the technique of Coupling from the Past.
However, the sample complexity of this algorithm depends on the structure of
the distribution $\D$ and can be even exponential in the support of $\D$ in many natural scenarios. Our main contribution is to provide an efficient exact sampling algorithm whose complexity does not
depend on the structure of $\D$. To this end, we give a parametric Markov chain that mixes significantly faster given a good approximation to the stationary distribution. 
We can obtain such an approximation using
an efficient learning from pairwise comparisons algorithm (Shah \emph{et al.}, JMLR 17, 2016).
Our technique for speeding up sampling from a Markov
chain whose stationary distribution is approximately known is simple, general and possibly of independent interest.
\end{abstract}



\section{Introduction}
\label{sec:intro}
Machine Learning questions dealing with pairwise comparisons (PCs) constitute a fundamental and broad research area with multiple real-world applications \cite{shah2016estimation,kazai2011search,luengo2012crowdsourcing,shah2016stochastically,furnkranz2010preference,wauthier2013efficient}. For instance, the
preference of a consumer to choose one product over another constitutes a pairwise comparison
between the two products. The two bedrocks of modern ML applications are inference/learning and sampling. The area of learning from pairwise comparisons has a vast amount of work \cite{rajkumar2014statistical, shah2016stochastically, dwork2001rank, negahban2012iterative, hunter2004mm,jiang2011statistical,agarwal2018accelerated,li2021performance,hajek2014minimax, shah2016estimation, vojnovic2016parameter, negahban2017rank, vojnovic2020convergence}.
In this problem, there exist $n$ alternatives
and an unknown weight vector $\D \in \reals^n$ where $\D(i)$ is the weight of the $i$-th element with  $\D(i) \geq 0$ and $\sum_{i \in [n]} \D(i) = 1$, i.e., $\D$ is a probability distribution over $[n]$. Roughly speaking, the learner observes pairwise comparisons (or more generally $k$-wise comparisons) of the form $(w, \{u,v\}) \in [n] \times [n]^2$ that compare the alternatives $u \neq v$ and $w$ equals either $u$ or $v$ indicating the winner of the comparison. The most fundamental distribution over pairwise comparisons is the BTL model \cite{BradleyTerry, Luce} where the probability that the learner observes $w = u$ when the items $u,v$ are compared is $\D(u)/(\D(u) + \D(v))$. The learner's goal is mainly to estimate the underlying distribution $\D$ using a small number of pairwise comparisons.

In this work, we focus on questions concerning sampling, the second bedrock of modern ML. In fact, we will deal with a fundamental area of the sampling literature, namely \emph{perfect sampling}. In contrast to approximate sampling, perfect (or exact) sampling from a probabilistic model requires an exact sample from this model. Problems in this vein of research lie in the intersection of computer science and probability theory and have been extensively studied \cite{propp1996exact,propp1998get, huber2004perfect,huber2016perfect,feng2019perfect,jain2021perfectly,anand2021perfect,guo-lll,bhandari2020improved,he2021perfect}. Our work poses questions concerning perfect sampling in the area of pairwise comparisons. We ask the following simple yet non-trivial question: Is it possible to obtain a perfect sample $y$ distributed as in $\D$ by observing only pairwise comparisons from $\D$?

Apart from the algorithmic and mathematical challenge of generating exact samples, perfect simulation has also practical motivation, noticed in various previous works \cite{huber2016perfect,jain2021perfectly,anand2021perfect}. Crucially, its value does not stem from the fact that the output of the sampling distribution is perfect; usually, one can come up with a Markov chain whose evolution rapidly converges to the stationary measure. Hence, the distance between the (approximate) sampling distribution from the desired one can become small efficiently. This approximate sampling approach has a clear drawback: there is no termination criterion, i.e., the Markov chain has to be run for
sufficiently long time and this requires an a priori known bound on the chain's mixing time, which may be much larger than necessary.
On the contrary, perfect sampling algorithms come with a stopping rule and this termination condition can be attained well ahead of the worst-case analysis time bounds in practice.

Second, the question that we pose appears to have links with the important literature of truncated statistics \cite{Galton1897,DGTZ18}.
Truncation is a common biasing mechanism under which samples from an underlying distribution $\D$ over $\mathbb{X}$ are only observed if they fall in  a given set $S \subseteq \mathbb{X}$, i.e., one only sees samples from the conditional distribution $\D_S$.
There exists a vast amount of applications that involve data censoring (that go back to at least  \cite{Galton1897}) where one does not have direct sample access to $\D$. Recent work in this literature \cite{fotakis2020efficient} asked when it is possible to obtain perfect samples from some specific truncated discrete distributions (namely, truncated Boolean product distributions). The area of pairwise comparisons can be seen as an extension of the truncated setting where the truncation set can change (i.e., each subset of alternatives corresponds to some known truncation set); 
while perfect sampling has been a subject of interest in truncated statistics, the neighboring area of pairwise comparisons lacks of efficient algorithms for this important question. 

\paragraph{Problem Formulation.} Let us define our setting formally (which we call Local Sampling Scheme). This setting is standard (without this terminology) and can be found e.g., at  \cite{rajkumar2014statistical}.

\begin{definition}
[see \cite{rajkumar2014statistical}]
\label{def:lss}
Let $\mathcal{Z}$ be a finite discrete domain. Consider a target distribution $\D$ supported on $\mathcal{Z}$ and a distribution $\Q$ supported on subsets of $\mathcal{Z}$. The sample oracle $\mathrm{Samp}(\Q;\D)$ is called \emph{Local Sampling Scheme (LSS)} and each time it is invoked, it returns an example $(x, S)$ such that: (i) the set $S  \subseteq \mathcal{Z}$ is drawn from $\Q$ and (ii) $x$ is drawn from the conditional distribution $\D_S$, where $\D_{S}(x) = \D(x) \vec 1\{x \in S\}/\D(S)$ and $\D(S) = \sum_{x \in S}\D(x)$.
\end{definition}

For simplicity, we mostly focus on the case where $\Q$ is supported on a subset of $\mathcal{Z} \times \mathcal{Z}$, and often refer to $\Q$ as the \emph{pair distribution}. Then, $\Q$ naturally induces an (edge weighted) undirected graph $G_{\Q}$ on $\mathcal{Z}$, where $\{u, v\} \in E(G_{\Q})$, if $\{u, v\}$ is supported on $\Q$, and has weight equal to the probability $\Q(u, v)$ that $\{ u, v \}$ is drawn from $\Q$. When we deal with a pair distribution $\Q$, the above generative model lies in the heart of the pairwise comparisons  \cite{FlignerV1986,marden1996analyzing, wauthier2013efficient}.
The comparisons are generated according to the  Bradley-Terry-Luce (BTL) model (where the comparisons have size $2$, \cite{BradleyTerry, Luce}): there is an unknown weight vector $\D \in \reals_+^n$ ($\D(u)$ indicates the quality of item $u$) and, for two items $u,v$,
the algorithm observes that $u$ beats $v$ with probability $\frac{\D(u)}{\D(u) + \D(v)}$. Motivated by the importance of perfect simulation, we ask the following question:
\begin{center}
    \emph{\color{red}\texttt{(Q):}\color{black} ~Is there an efficient algorithm that draws i.i.d. samples $(x,S)$ from $\mathrm{Samp}(\Q;\D)$ and generates a single sample $y \sim \D$?
}
\end{center}

To put our contribution into context, let us ask another question whose answer is clear from previous work. What is the sample complexity of \emph{learning} $\D$ from a Local Sampling Scheme $\mathrm{Samp}(\Q;\D)?$
Here, the goal is to estimate $\D$ in some $L_p$ norm given such comparisons. Some works focus on learning the re-parameterization $\vec z \in \reals^n$ with $z_x = \log(\D(x))$. Depending on the context, either the normalization condition $\<\vec 1,\D\> = 1$ or $\<\vec 1, \vec z\> = 0$ is used.
\cite{shah2016estimation}
deal with the problem of learning the re-parameterization $\vec z$ in the $L_2$ norm when
$|S| = 2$ (pairwise comparisons) or $|S| = k$ ($k$-wise comparisons with $k = O(1)$).
The sample complexity for learning $\vec z$ in $L_2$ norm from pairwise comparisons is $n/(\lambda(\vec Q) \epsilon^2)$, where $\lambda(\vec Q)$ is the second smallest eigenvalue of a Laplacian matrix $\vec Q$, induced by the pair distribution $\Q$ (the learning result is tight in a minimax sense for some appropriate semi-norm \cite{shah2016estimation}). 
In particular, the matrix $\vec Q$ is defined as $\vec Q_{xy} = -\Q(x,y)$ and $\vec Q_{xx} = \sum_{y \neq x} \Q(x,y)$.
Hence, in the pairwise comparisons setting, the sample complexity of learning incurs an overhead associated with $\lambda(\vec Q)$ (Table \ref{table:problems1}).
The algorithm of \cite{shah2016estimation} is a pivotal component for our main result, as we will see later. \cite{agarwal2018accelerated} provide a similar result for learning $\D$ in TV distance using a random-walk approach. We remark that learning from pairwise comparisons requires some mild conditions concerning the sampling process (LSS in our case). We review them shortly below.

\begin{table}[ht]
    \centering
   \begin{tabular}{ |p{3.9cm} | p{5cm} | p{4.2cm}|  }
 \hline
 \textbf{Sample Access} & \textbf{Learning $(\epsilon > 0)$ of $\vec z$ in $L_2$} & \textbf{Exact Sampling}  \\
 \hline \\[-1.16em]
 \Cref{def:lss} with PCs & 
 $O \left( n/ \epsilon^2 \right) \cdot \frac 1 {\lambda(\vec Q)}$ ~~ \cite{shah2016estimation} 
& \color{blue}$ \wt O(n^2) \cdot \frac 1 {\lambda(\vec Q)}$ \color{black} (\Cref{thm:sampling})\\ 
 \hline
\end{tabular}
    \caption{Learning and Exact Sampling from PCs of size 2. $\wt{O}(\cdot)$ subsumes logarithmic factors. }
    \label{table:problems1}
\end{table}

\paragraph{Conditions for Local Sampling Schemes.}
We provide a standard pair of conditions for Local Sampling Schemes (these or similar conditions are also needed in the learning problem). We are interested in LSSs that satisfy both of these conditions. However, some of our results still hold
when only the first condition is true; this will be more clear later.
\Cref{assumption:identifiability} 
is a necessary information-theoretic condition for the pair distribution $\Q$.
\begin{assumption}[Identifiability]
\label{assumption:identifiability}
The support of the pair distribution $\Q$ of the Local Sampling Scheme contains a spanning tree, i.e., the induced graph $G_{\Q}$ is connected. We define $\calE := E(G_{\Q}).$
\end{assumption}
Intuitively, the pair distribution $\Q$
needs to be supported on a spanning tree. Equivalently, the associated Laplacian matrix $\vec Q$ should have positive Fiedler eigenvalue $\lambda(\vec Q) > 0$. 
\Cref{assumption:efficsample} is a condition about the support of $\Q$ \emph{and} the target distribution $\D$; it allows efficient learning of $\D$ from pairwise comparisons and is related to the difficulty of estimating small $\D$-probabilities.
\begin{assumption}[Efficiently Learnable]
\label{assumption:efficsample}
There exists a constant $\phi > 1$ such that target distribution $\D$ satisfies 
$ \frac{1}{\phi} \leq \max_{(x,y) \in \calE} \frac{\D(x)}{\D(y)} \leq \phi\,,$ where $\calE$ is the support of the distribution $\Q$.
\end{assumption}
From a dual perspective, \Cref{assumption:efficsample} says that $\Q$ is supported only on edges where the two corresponding $\D$-ratios are well controlled and implies that any conditional distribution $\D_{\{x,y\}}$ has bounded variance, i.e., $(1/\phi)/(1+\phi)^2 \leq \frac{\D(x)\D(y)}{(\D(x) + \D(y))^2} \leq \phi / (1 + 1/\phi)^2$. A quite similar property has been proposed in the problem of learning from pairwise comparisons~\cite{shah2016estimation}.

\smallskip\noindent\textbf{Contributions and Techniques.}
We provide an algorithmic answer to question \emph{\color{red}\texttt{(Q)}\color{black}} regarding perfect sampling (under the above mild conditions) by showing the following:
\begin{center}
    \emph{There exists an efficient algorithm that draws $\wt{O}(n^2 / \lambda(\vec Q))$ samples from $\mathrm{Samp}(\Q;\D)$ and outputs a perfect sample $y \sim \D$.}
\end{center}
The tool that makes this possible (and algorithmic) is a novel variant of the elegant Coupling from the Past (CFTP) algorithm of~\cite{propp1996exact}. CFTP is a technique developed by Propp and Wilson~\cite{propp1996exact, propp1998get}, that provides an exact random sample from a Markov chain, that has law the (unique) stationary distribution (we refer to \Cref{appendix:exact-sampling-prelims} for an introduction to CFTP). To make our contribution clear, we provide two theorems; \Cref{thm:inf:naive} that directly and ``naively'' applies CFTP (indicating the challenges behind efficient sample correction) and a more sophisticated \Cref{thm:inf:sampling} that \emph{combines CFTP and learning from pairwise comparisons} in order to get the result of \Cref{table:problems1}. We show how to efficiently generate ``global'' samples from the true distribution $\D$ from few samples from a Local Sampling Scheme $\mathrm{Samp}(\Q;\D)$. To this end, our novel algorithm combines the CFTP method \cite{propp1996exact, propp1998get}, a remarkable technique that generates exact samples from the stationary distribution of a given Markov chain, with (i) an efficient algorithm that estimates the parameters of Bradley-Terry-Luce ranking models from pairwise (and, in general, $k$-wise) comparisons \cite{shah2016estimation} and (ii)  procedures based on Bernoulli factories \cite{dughmi2017bernoulli,mossel2005new,nacu2005fast} and rejection sampling. 

At a technical level, we observe (\Cref{s:naive}) that a Local Sampling Scheme $\mathrm{Samp}(\Q;\D)$ naturally induces a canonical  ergodic Markov chain on $[n]$
with transition matrix $\vec M$. The probability of a transition from a state $x$ to a state $y$ is equal to the probability $\Q(x, y)$, that the pair $\{ x, y\}$ is drawn from $\Q$, times the probability $\D(y)/(\D(x)+\D(y))$, that $y$ is drawn from the conditional distribution $\D_{\{x, y\}}$. Then, the unique stationary distribution of the resulting Markov chain coincides with the true distribution $\D$. Applying (a non-adaptive version of) the Coupling from the Past algorithm, 
we obtain the following result under  \Cref{assumption:identifiability}
and
\Cref{assumption:efficsample}
for a target distribution $\D$ over $[n]$. For the matrix $\vec M$, we let $\Gamma(\vec M)$ be its absolute spectral gap (see \Cref{app:notation} for extensive notation). 
\begin{inftheorem}[Direct Exact Sampling from LSS]\label{thm:inf:naive}
With an expected number of 
$\wt{O} \left(\frac{n^2 }{\Gamma(\vec{M})
}\right)$ samples from a Local Sampling Scheme $\mathrm{Samp}(\Q;\D)$, there exists an algorithm (\Cref{algo:exact-naive}) that generates a sample distributed as in $\D$.  
\end{inftheorem}
The sample complexity of \Cref{thm:inf:naive} is, for instance, attained by the instance of \Cref{fig:bad_instance}, which we will discuss shortly\footnote{For the instance of \Cref{fig:bad_instance} one can reduce the sample complexity of our results by a factor of $n$, since the state space has a natural ordering and one could apply monotone CFTP (it suffices to coalesce two random processes and not $n$).}. We remark that \Cref{thm:inf:naive} holds more generally for Local Sampling Schemes that only satisfy \Cref{assumption:identifiability} and not \Cref{assumption:efficsample}; for a formal version of the above result, we refer to \Cref{theorem:naive} and for a discussion on the resulting sample complexity, we refer also to \Cref{remark:sample-complexity-naive}. Let us now see why
\Cref{thm:inf:naive} is quite unsatisfying.

In \Cref{thm:inf:naive}, the matrix $\vec{M}$  contains the transitions induced by $\mathrm{Samp}(\Q;\D)$ and the transition $u \to v$ is performed with probability $\Q(u,v) \D(v)/(\D(u) + \D(v))$ and $\Gamma(\vec{M})$ is the absolute spectral gap of the Markov chain's transition matrix $\vec{M}$ on which CFTP is performed. 
Note that both $\vec{M}$ and $\Gamma(\vec{M})$ depend on $\Q$ 
and the target distribution $\D$. In many natural cases, e.g., if $\D$ is multimodal or includes many low probability points, the CFTP-based algorithm of \Cref{thm:inf:naive} can be quite inefficient! For instance, \Cref{fig:bad_instance} is a bad instance for the standard CFTP algorithm. It corresponds to an instance where the support of $\Q$ is the path graph over $[n]$ and the target distribution $\D$ is bimodal, satisfying \Cref{assumption:identifiability} and \Cref{assumption:efficsample} with $\phi = 2$. In this case, $\vec M$ captures the transitions of a negatively biased nearest-neighbor random walk on the path, i.e., the probability of moving to the boundary is larger than moving to the interior. The coalescence probability of the two extreme points is of order $\alpha^n$ for some $0 < \alpha < 1$. This probability is connected to the maximum hitting time and, consequently, we get that the mixing time is $\Omega(\alpha^{-n})$ \cite{levin2017markov} and so the coalescence time is exponentially large. \Cref{thm:inf:naive} depends on $\Gamma(\vec M)$ and hence the \emph{exponential dependence on} $n$ is reflected in the conductance of the chain, which is exponentially small. So, can we obtain perfect $\D$-samples using the oracle of \Cref{def:lss} efficiently?

\tikzset{
  bar/.pic={
    \fill (-.1,0) rectangle (.1,#1) (0,#1) node[above,scale=1/2]{$#1$};
  }
}

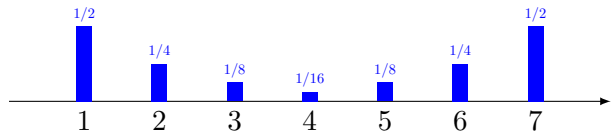
\begin{figure}[ht]
    \centering
    \begin{tikzpicture}[y=2cm]
    \draw
      (0,0) edge[-latex] (8,0)
      foreach[count=\i from 1] ~ in {1/2, 1/4, 1/8, 1/16, 1/8, 1/4, 1/2}{
        (\i,0) pic[blue]{bar=~} node[below]{$\i$}
      };
  \end{tikzpicture}
    \caption{A bad instance for \Cref{thm:inf:naive}.}
    \label{fig:bad_instance}
\end{figure}



As a technical contribution, we provide a novel exact sampling algorithm whose convergence does not depend on $\D$. We show how to improve the efficiency of CFTP by removing the dependence of its sample (and computational) complexity on the structure of the true distribution $\D$. In particular, in \Cref{section:analysis-algo}, we show the following under \Cref{assumption:identifiability} 
and \Cref{assumption:efficsample}: 

\begin{inftheorem}[Exact Sampling from LSS using Learning] \label{thm:inf:sampling}
There exists an algorithm (\Cref{algo:exact} \& \Cref{algo:total-exact}) that draws an expected number of
$\wt{O}\left( \frac{ n^2}{ \lambda(\vec Q) } \right)$
samples from a Local Sampling Scheme $\mathrm{Samp}(\Q;\D)$, and generates a sample distributed as in $\D$. 
\end{inftheorem}
In the above, $\lambda(\vec{Q})$ is the second smallest eigenvalue (a.k.a., the Fiedler eigenvalue) of $\vec{Q}$, the weighted Laplacian matrix of the graph $G_{\Q}$ induced by $\Q$; this
is the same matrix that appears in the learning result of Table \ref{table:problems1}. 
\emph{Crucially, the sample complexity in the above result is polynomial in $n$}, provided that the pair distribution $\Q$ is reasonable.
For the path graph of \Cref{fig:bad_instance} with $\Q$ the uniform distribution,
\Cref{thm:inf:naive} yields an exponential sample complexity but the spectrum of the matrix of \Cref{thm:inf:sampling}
is
$\mathrm{spec}(\vec{Q}) = \{ \frac{1}{n}\Theta(1 - \cos(\pi i / n)) \}_{i=0}^{n-1}.$ Hence, since $1-\cos(x) = \sin^2(x/2) \approx x^2$, we get that the sample complexity is $\poly(n).$
We also note that the runtime of the algorithm is sample polynomial.

\paragraph{Technical Overview.} A first key idea towards establishing \Cref{thm:inf:sampling}
is to exploit \emph{learning} in order to
accelerate the convergence of CFTP. We first apply (a modification of) the gradient-based algorithm of \cite{shah2016estimation} to the empirical log-likelihood objective, which estimates the parameters of BTL ranking models from PCs, with samples from $\mathrm{Samp}(\Q;\D)$.
As a result of this learning algorithm (whose efficiency requires \Cref{assumption:efficsample}, as in \cite{shah2016estimation}), we get a probability distribution $\wt{\D}$, which approximates $\D$ with small relative error, in the sense that $\D(x) \in (1\pm\epsilon) \cdot \wt{\D}(x)$, for any $x$ in $\D$'s support. As we saw in \Cref{table:problems1}, the sample complexity will be roughly of order $O(n/(\lambda(\vec Q) \epsilon^2)))$. To be precise, this learning phase gives us a sample complexity term of order $O(n^2/\lambda(\vec Q) )$ and corresponds to a single execution of the above learning algorithm for the target $\D$ in relative error with accuracy $\epsilon = \Theta(1/\sqrt{n})$ (the accuracy's choice will be more clear later).

The new idea is that we can use the output of the learning algorithm $\wt{\D}$ and a Bernoulli factory mechanism to transform the Markov chain induced by $\mathrm{Samp}(\Q;\D)$ into an ergodic chain with almost uniform stationary probability distribution, whose transition probabilities (and mixing time) essentially do not depend on $\D$! To do so, for each pair of states $x, y$, with $(x, y) \in \calE$ and $\wt{\D}(x) > \wt{\D}(y)$, we can use Bernoulli downscaling (which constitutes an essential building block for general Bernoulli factories, see e.g.,~\cite{dughmi2017bernoulli,mossel2005new,nacu2005fast}) to transform the (unknown) probability of a transition from $y$ to $x$ to $\D(x)/(\D(x)+\D(y)) \cdot \wt{\D}(y)/\wt{\D}(x)$ (conditional that the edge $(x,y)$ is drawn). Since $\wt{\D}(x) \approx \D(x)$, for all states $x$, this makes the probability of a transition from $x$ to $y$ almost equal to the probability of a transition from $y$ to $x$, which, in turn, implies that \emph{the stationary distribution of the modified Markov chain is almost uniform}.
Thus, we speed up the CFTP algorithm by removing the modes due to the landscape of the target $\D$  and can get an exact sample from the stationary distribution of the modified Markov chain with 
$O\left(n \log (n) / \lambda(\vec Q)\right)$ 
samples. Namely, the sample complexity of this parameterized variant of the CFTP algorithm does not depend on the structure of $\D$ anymore (of course, it does depend on the size $n$ of $\D$'s support). The above constitute our second key idea and the intuition behind this re-weighting comes from the method of simulated tempering \cite{marinari1992simulated, ge2018simulated, ge2018beyond}.

Our third idea deals with the output of the algorithm. Since we want an exact/clean sample from $\D$, we cannot simply print the output of the above process. As a last step, we use rejection sampling, guided by the distribution $\wt{\D}$, and apply the ``reverse'' modification, bringing the distribution of the resulting sample back to $\D$ (see \Cref{algo:exact}). With high probability, this last rejection sampling will succeed after $\Theta(n)$ executions of the parameterized CFTP algorithm. However, note that we do not have to execute the learning algorithm again.
Hence, the total sample complexity is equal to the cost of the learning phase (that is performed only once and its output is given as input to the parameterized CFTP algorithm; see \Cref{algo:exact}) and the cost of the modified CFTP multiplied by $\Theta(n)$ (due to rejection sampling), which yields a total of $\wt{O}(n^2/\lambda(\vec Q))$ samples.
The above modification of the Markov chain induced by $\mathrm{Samp}(\Q;\D)$, which improves possibly significantly its mixing time, is simple and general
and the parameterized variation of the standard CFTP algorithm (\Cref{algo:exact})
might be of independent interest. In \Cref{appendix:hypergraphs}, we discuss how to extend the analysis of the main \Cref{algo:exact} to sets of size larger than $2$.

\subsection{Related Work} 
Our results are related to and draw ideas from several areas of theoretical computer science:

\paragraph{Exact Sampling.} Exact sampling comprises a well-studied field \cite{kendall2005notes,huber2016perfect} with numerous applications in theoretical computer science (e.g., to approximate counting~\cite{huber1998exact}). One of our main tools, the Coupling From the Past algorithm\footnote{
We would like to shortly mention an interesting example that shows how perfect sampling from pairwise comparisons is possible from an algebraic perspective which is different from the standard random-walk viewpoint: consider the case where $n = 3$ and the discrete distribution $\D$ is the vector $(\D(a), \D(b), \D(c))$. Assume that $\Q$ is supported on $\{a,b\}$ and $\{b,c\}$ with equal probability. The following process generates an exact sample from $\D$: 
given a sample $x$ from $\D_{\{a,b\}}$ and $y$ from $\D_{\{b,c\}}$, we repeat until $(x,y) \neq (a,c)$. In this case, $b$ will be observed in at least one coordinate and the algorithm outputs the other coordinate. We claim that this process outputs a perfect sample from $\D$.
An algebraic intuition behind this random process is that $
    (a+b)(b+c) = b(a+b+c) + ac$, where $+$ corresponds to the logical or; the bad event is the term $ac$, while the good event contains $b$ and the perfect sample ($a$ or $b$ or $c$).
This is the generating polynomial of the above sampling process. Formally, the probability that this process outputs $a$ is $
    \Pr[a] \propto \frac{\D(a)}{\D(a) + \D(b)} \frac{\D(b)}{\D(b) + \D(c)} \sum_{i=0}^{\infty} \left (\frac{\D(a)}{\D(a) + \D(b)} \frac{\D(c)}{\D(b) + \D(c)} \right)^i$ and note that $ \Pr[x] \propto \D(x)$
for any $x \in \{a,b,c\}$ and all have the same normalization constant. This verifies our claim.
}, is the breakthrough result of~\cite{propp1996exact,propp1998get}, which
established the applicability of perfect simulation, 
and made efficient simulation of challenging problems (e.g., Ising model) feasible.
After this fundamental result, the literature of perfect simulation flourished and various exact sampling procedures have been developed, see e.g., \cite{propp1996exact,green1999exact,wilson2000couple,haggstrom2000propp,fill2000randomness,huber2004perfect}, and the references therein. 
\paragraph{Learning and Ranking from Pairwise Comparisons.} There is a vast amount of literature on ranking and parameter estimation from pairwise comparisons (see e.g., \cite{FlignerV1986,marden1996analyzing,cattelan2012models} and the references therein). A significant challenge in that field is to come up with models that accurately capture uncertainty in pairwise comparisons. The standard ones are the Bradley-Terry-Luce (BTL) model~\cite{ BradleyTerry,Luce} and the Thurstone model~\cite{thurstone1927law}. Our work is closely related to previous work on the BTL model \cite{rajkumar2014statistical, shah2016stochastically, dwork2001rank, negahban2012iterative, hunter2004mm,jiang2011statistical,agarwal2018accelerated,li2021performance} and
our sample complexity rates depend on the Fiedler eigenvalue of a pairwise comparisons' matrix, as in \cite{hajek2014minimax, shah2016estimation, vojnovic2016parameter, negahban2017rank, vojnovic2020convergence}. For more details on ranking distributions, see e.g., \cite{liu2018efficiently,busa2019optimal,fotakis2021aggregating,fotakis22a}.

\paragraph{Truncated Statistics.} Our work falls in the research agenda of efficient learning and exact sampling from truncated samples. Recent work has obtained efficient algorithms for efficient learning from truncated \cite{DGTZ18,KTZ19,DGTZ19,ilyas2020theoretical, daskalakis2020truncated,bhattacharyya2020efficient,NP19,fotakis2020efficient,daskalakis2021statistical} and censored \cite{moitra21a,fotakis2021efficient,plevrakis21a} samples. Closer to our work, \cite{fotakis2020efficient} also deal with the question of exact sampling from truncated samples.
\paragraph{Correction Sampling.} Sample correction is a field closely related to our work. This area of research deals with the case where input data are drawn from a noisy or imperfect source and the goal is to design \emph{sampling correctors}~\cite{canonne2018sampling} that operate as filters between the noisy data and the end user. These algorithms exploit the structure of the distribution and make ``on-the-fly'' corrections to the drawn samples. 
Moreover, the problem of local correction of data has received much attention in the contexts of self-correcting programs, locally
correctable codes, and local filters for graphs and functions (see e.g., \cite{blum1993self, saks2010local, jha2013testing, ailon2004property, bhattacharyya2010lower, gergatsouli2020black}).
\paragraph{Conditional Sampling.} Conditional sampling deals with the \emph{adaptive} learning analog of LSS, where the goal is again to learn an underlying discrete distribution from conditional/truncated samples, but the learner can change the truncation set on demand; see e.g.,~\cite{Canonne15b,FalahatgarJOPS15,AcharyaCK15b, BhattacharyyaC18,AcharyaCK15a,GouleakisTZ17,KamathT19, canonne2019random}. 

\subsection{Notation}
\label{app:notation}

In this section, we provide the basic notation we are going to use in the technical part.

\paragraph{General Notation.} We define $[n] := \{1,\dots,n\}$.
We denote vectors with small bold letters $\vec x$ and matrices with bold letters $\vec{P} = [\vec P_{xy}]$.
For a vector $\vec x$, we let $x_i = x(i)$ be its $i$-th entry.
For $p \in \{1,2,\infty\}$, we denote the $L_p$ norm by $\|.\|_p$. For a matrix $\vec A$, let $\| \vec A \|_2$ denote its spectral norm (largest singular value). We consider graphs $G = (V,E)$, with $|V| = n$ vertices, whose associated Laplacian matrix is denoted $\vec \Lambda(G)$.
The edge set of the graph $G$ is the set $E(G)$.
In \Cref{appendix:prelims}, we have included some basic definitions and facts about random walks and Markov chains. We define the weighted Laplacian matrix of $G_{\Q}$ for a distribution $\Q$ as
\begin{equation}
    \label{eq:laplacian-q}
    \vec Q_{xy} = -\Q(x,y) ~ \text{for any $x \neq y$ and } ~ \vec Q \vec 1 = \vec 0\,,
\end{equation}

\paragraph{Distributions and Distances.} 
The probability simplex is denoted by $\Delta^n$.
We let $\supp(\D)$ denote the support of a distribution $\D$. The total variation distance between two discrete distributions $\D_1,\D_2$ over $[n]$ is equal to $\mathrm{TV}(\D_1,\D_2) = \max_{S \subseteq [n]} (\D_1(S) - D_2(S)) = \frac{1}{2} \|\D_1-\D_2\|_1$. We let $\D(S) = \sum_{x \in S}\D(x)$
and $\D_S$ is the conditional distribution on the set $S$, i.e., $\D_S(x) = \frac{\D(x)}{\D(S)}\vec 1\{ x \in S \}.$ The vector $\vec z \in \reals^n$ with $z_i = \log(\D(i))$ is called the natural parameter vector of $\D$.

%

%

\paragraph{Random Walks.} For a reversible transition matrix $\vec{P}$, let its (real) spectrum be $1 = \lambda_1 \geq \lambda_2 \geq \dots \geq \lambda_n \geq -1$. We define the \emph{absolute spectral gap} of $\vec{P}$ to be the difference $\Gamma(\vec{P}) = 1 - \max\{ \lambda_2(\vec{P}), |\lambda_n (\vec{P})| \}$. If $\vec{P}$ is aperiodic and irreducible, then $\Gamma(\vec{P})> 0$. One could also define the \emph{spectral gap} $1-\lambda_2(\vec{P})$. The two gaps are equal when the chain is lazy, i.e., when $\forall i \in [n]$, $\vec{ P}_{ii}\geq 1/2$. We 
let $\lambda(\vec{L})$ be the second \emph{smallest} eigenvalue of a Laplacian matrix $\vec{L}$, a.k.a., its Fiedler eigenvalue. The Laplacian matrix $\vec L \in \reals^{n \times n}$ is positive semi-definite and induces a semi-norm
on $\reals^n$ with $\| \vec v \|_L := \sqrt{\vec v^T \vec L \vec v}$ for $\vec v \in \reals^n$.  Recall that a semi-norm differs from a norm in that the semi-norm of a non-zero element is allowed to be zero. In \Cref{algo:exact-naive}, \Cref{algo:exact} and \Cref{algo:exact-hypergraph}, we use the notation $F_t(x) \gets F_{t+1}(w)$ for some $t < 0$. This means that we append the path of $x$ at time $t$ by adding the transition $x \to w$ and then the path of $w$ from time $t+1$ up to $0$. 

\paragraph{Matrix Operators.} We denote by $\vec{A} \odot \vec{B} = (\vec A_{ij} \vec B_{ij})_{1 \leq i \leq j \leq n}$ the standard Hadamard matrix product and by $\vec A \circ \vec B$ a variation of the Hadamard
matrix product, where the off-diagonal entries are equal to those  of the standard Hadamard product $\vec{A} \odot \vec{B}$, but the diagonal entries correspond to the diagonal matrix with entries $(-\sum_{j \neq i} \vec A_{ij} \vec B_{ij} )_{i \in [n]}$. 
Finally, we let $\vec A \otimes \vec B = \vec I - \vec A \circ \vec B.$

\section{An Exact Sampling Algorithm using CFTP}\label{s:naive}

Consider a discrete target distribution $\D$. For pair distributions $\Q$, the Local Sampling Scheme $\mathrm{Samp}(\Q; \D)$
can be regarded as a graph $G = (V,E)$ with $V = \supp(\D)$ and $\calE = \supp(\Q)$. We let $V = [n]$. Let $\vec{M}$ denote the canonical transition matrix of the Markov chain, associated with the oracle $\mathrm{Samp}(\Q;\D)$. 
The entries of $\vec{M} = [\vec M_{xy}]_{x,y \in [n]}$ are defined as:
\begin{equation}
\label{equation:trans-matrix}
\vec M_{xy} = \Q(x,y)\frac{\D(y)}{\D(x) + \D(y)} 
~\textnormal{for any $x \neq y$ and }~\vec M_{xx} = 1 - \sum_{y \neq x} \vec M_{xy}~\textnormal{ for } x \in [n]\,.
\end{equation}

Observe that the transition from $x$ to $y$ is performed when both the edge $\{x,y\}$ is drawn from $\Q$, with probability proportional to $\Q(x,y)$
and $y$ is drawn from the conditional distribution $\D_{\{x,y\}}$\,.
The Markov chain, whose transitions are described by $\vec M$, has some notable properties: it has $\D$ as stationary distribution and is ergodic. We can invoke the fundamental CFTP algorithm over the Markov chain of \cite{propp1996exact,propp1998get} to obtain a perfect sample from $\D$. We extensively discuss CFTP in \Cref{appendix:exact-sampling-prelims} for the interested reader.
CFTP yields the following (unsatisfying) result, whose proof can be found at the \Cref{appendix:proof-naive}.

\begin{algorithm}[ht] 
\caption{Exact Sampling from Local Sampling Schemes of \Cref{thm:inf:naive} \& \Cref{theorem:naive}}
\label{algo:exact-naive}
\begin{algorithmic}[1]
\Procedure{ExactSampler()}{} \Comment{\emph{Sample access to the LSS oracle $\mathrm{Samp}(\Q; \D)$.}}
\State $t \gets 0$
\State $F_0(x) \gets x$, for any $x \in [n]$
\State $\textbf{while}~F_{t}~\text{has not coalesced}~\textbf{do}$ \Comment{\emph{While no coalescence has occured.}}
\State $~~~~~~~~ t \gets t - 1$
\State ~~~~~~~~Draw sample $(i,j,w)$ with $(i,j) \in \calE$ and $w \in \{i,j\}$ 
\State ~~~~~~~~\textbf{for}~$x = 1\ldots n$~\textbf{do}\Comment{\emph{In order to update state $x$.}}
\State ~~~~~~~~~~~~~~~~ \textbf{if} $x \notin \{i,j\}~\mathrm{or}~x = w$ \textbf{then} $F_{t}(x) \gets F_{t+1}(x)$
\Comment{\emph{Append $F_t$ in the past.}}
\State ~~~~~~~~~~~~~~~~ \textbf{else} $F_{t}(x) \gets F_{t+1}(w)$ \Comment{\emph{New transition with probability $\Q(x,w) \frac{\D(w)}{\D(x)+\D(w)}$.}} 
\State ~~~~~~~~\textbf{end} 
\State $\textbf{end}$
\State Output $F_{t}(1)$
\Comment{\emph{Output the perfect sample.}}
\EndProcedure

\end{algorithmic}
\end{algorithm}

\begin{theorem}[Direct Exact Sampling from LSS]\label{theorem:naive}
Let $\D_{\mathrm{min}}$ be the minimum value of the target distribution $\D$.
Under \Cref{assumption:identifiability} and for any $\delta > 0$, \Cref{algo:exact-naive} draws, with probability at least $1-\delta$, $N = {O} \left(\frac{n \log (1/\D_{\mathrm{min}})}{\Gamma( \vec{M}) } \log(\frac{1}{\delta})\right)$ samples from a Local Sampling Scheme $\mathrm{Samp}(\Q;\D)$ over $[n]$, runs in time polynomial in $N$ and outputs a sample distributed as in $\D$. 
\end{theorem}

Note that in the above result (\Cref{theorem:naive}), the confidence parameter $\delta$ corresponds to the number of samples required and not the quality of the output sample (the sample is perfect with probability $1$). We comment that the proof of the above Theorem is nearly straightforward and follows from the analysis of the CFTP algorithm; the purpose of \Cref{theorem:naive} is to indicate that a direct application of CFTP would potentially lead to an exponential number of samples due to the shape of the target distribution. Our main technical
contribution begins in the next section: We give an exact sampler that surpasses the dependence on the target’s structure using learning and rejection
sampling.

\section{Improving the Exact Sampling Algorithm using Learning}
\label{subsection:algo-description}

The main drawback of Theorem's \ref{theorem:naive} algorithm is that its sample complexity depends on both $\vec Q$ and $\vec D$ and consequently in many natural scenaria the sample complexity may be exponentially large in $n$. Next, we present \Cref{algo:exact}, which will allow us to remove the dependence on $\vec D$ for the Markov chain and reduce the sample complexity to polynomial in the domain size. 
The crucial differences compared to \Cref{algo:exact-naive} are the \color{orange} \textbf{orange} \color{black} lines. \Cref{algo:exact} is a parameterized extension of CFTP and its input is a vector  $\vec p \in [0,1]^n$. We remark that when $\vec p$ is the all-ones vector $\vec 1$, we obtain the previous \Cref{algo:exact-naive}. 
\begin{algorithm}[ht] 
\caption{Parameterized Extension of \Cref{algo:exact-naive}}
\label{algo:exact}
\begin{algorithmic}[1]
\State \color{orange} \textbf{Input:} $\vec p \in [0,1]^n$ \color{black} \Comment{\emph{For instance, this may be the output of \Cref{algo:learn}.}}
\Procedure{ExactSampler}{$\vec p$} \Comment{\emph{Sample access to the LSS oracle $\mathrm{Samp}(\Q; \D)$.}}
\State $t \gets 0$, $F_0(x) \gets x$, for any $x \in [n]$
\State $\textbf{while}~F_{t}~\textnormal{has not coalesced}~\textbf{do}$ \Comment{\emph{While no coalescence has occured.}}
\label{line:while}
\State $~~~~~~~~ t \gets t - 1$
\State ~~~~~~~ Draw sample $(i,j,w)$ with $(i,j) \in \calE$ and $w \in \{i,j\}$ 
\State ~~~~~~~~\textbf{for}~$x = 1\ldots n$~\textbf{do}\Comment{\emph{In order to update state $x$.}}
\State ~~~~~~~~~~~~~~~~ \textbf{if} $x \notin \{i,j\} ~\mathrm{ or }~ x = w$ \textbf{then} $F_{t}(x) \gets F_{t+1}(x)$
\State ~~~~~~~~~~~~~~~~ \textbf{else}
$
F_{t}(x) \gets \left\{\begin{array}{lr}
        \color{orange} F_{t+1}(w), & \color{orange} \textnormal{with probability } \min \{ p(x) / p(w),1 \} \color{black} \\
        \color{orange} F_{t+1}(x), & \color{orange} \textnormal{otherwise } \color{black}
        \end{array}\right.$
\label{line:scale} \State ~~~~~~~~\textbf{end} 
\State $\textbf{end}$
\label{line:end}
\State \color{orange} Draw $C \sim \mathrm{Be}(p(F_t(1)))$ \color{black} \Comment{\emph{$\mathrm{Be}(p)$ is a $p$-biased Bernoulli coin}.}
\label{line:ber}
\State \color{orange} \textbf{if} $C = 1$ \textbf{then} \color{black} Output $F_t(1)$ \color{orange} \textbf{else} Output $\perp$ \color{black}
\Comment{\emph{Output the perfect sample or Fail.}}
\label{line:reject}
\EndProcedure

\end{algorithmic}
\end{algorithm}

We perform \Cref{algo:exact} using as input parameter $\vec p$ an estimate $\wt{\D}$ of the target $\D$ with (relative) error of order $\epsilon = 1/\sqrt{n}$ (this estimate is obtained via \Cref{theorem:learning12} \& \Cref{algo:learn} using $O(n^2 /\lambda(\vec Q) )$ samples from $\mathrm{Samp}(\Q;\D)$, as we will see later). Under the perspective of LSSs as random walks on an irreducible Markov chain, \Cref{algo:exact} proceeds by executing CFTP. Recall that in each draw from $\mathrm{Samp}(\Q;\D)$, a transition can be realized by the matrix $\vec M$
. The fact that this transition depends on the target $\D$ is pathological, as we observed previously.
The algorithm, instead of performing this transition (as the non-adaptive CFTP algorithm of \Cref{theorem:naive} does), performs a Bernoulli factory mechanism (downscaling), using the estimates $\wt{\D}$ in order to make each transition almost uniform (see {Line \color{black} \ref{line:scale} \color{black}} of \Cref{algo:exact}). This change introduces some (known) bias to the random walk and changes the stationary distribution from $\D$ to an almost uniform one. By performing the CFTP method, the algorithm iteratively reconstructs the past of the infinite simulation, until all the simulations have coalesced at time $t = 0$. When coalescence occurs, the algorithm has an exact sample from the biased (almost uniform) stationary distribution. Since the introduced bias is known (it corresponds to the inverse of the estimates $\wt{\D}$), the algorithm accepts the sample using rejection sampling, guided by $\wt{\D}$, so that the final sample is distributed according to the target distribution $\D$ (see {Line \color{black} \ref{line:reject}\color{black}} of \Cref{algo:exact}). Specifically, we show (for the proof, see \Cref{appendix:proof-main}) that:
\label{section:analysis-algo}
\begin{theorem}[Exact Sampling from LSS using Learning]\label{thm:sampling}
Under \Cref{assumption:identifiability} and \Cref{assumption:efficsample}, for any $\delta > 0$, there exists an algorithm (\Cref{algo:total-exact}) that draws, with probability at least $1-\delta$, $N = {O}\left( n^2 \log^2(n)/\lambda(\vec Q) \cdot \log(1/\delta) \right)$
samples from a Local Sampling Scheme $\mathrm{Samp}(\Q;\D)$, runs in time polynomial in $N$, and outputs a sample distributed as in $\D$. 
\end{theorem}
\begin{algorithm}[ht] 
\caption{Exact Sampling from Local Sampling Schemes of \Cref{thm:inf:sampling} \& \Cref{thm:sampling}}
\label{algo:total-exact}
\begin{algorithmic}[1]
\Procedure{ExactSamplerWithLearning}{$\delta$} \Comment{\emph{The algorithm of \Cref{thm:sampling}}.}
\State $\wt{\D} \gets \textsc{Learn}(\epsilon := 1/\sqrt{n}, \delta)$ \Comment{\emph{Learn $\D$ in relative error with \Cref{algo:learn} (\Cref{theorem:learning12}).}}
\label{line:learn}
\State x $ \gets \perp$
\State \textbf{while} x $= \perp$ \textbf{repeat}
\label{line:while-2}
\State ~~~~~~~~ x $\gets$ \textsc{ExactSampler}$(\wt{\D})$ \Comment{\emph{Call \Cref{algo:exact}.}}
\label{line:loop}
\State Output x
\label{line:out}
\EndProcedure
\end{algorithmic}
\end{algorithm}
We proceed with the analysis of \Cref{algo:total-exact}, which can be decomposed into two parts: the Learning Step (Line \ref{line:learn}) and several iterations of \Cref{algo:exact} (Line \ref{line:loop}).
\paragraph{Learning Phase.} For two distributions $\D, \wt{\D}$ with ground set $[n]$, we introduce the sequence/list (of length $n$) $1 - \D/\wt{\D} := (1 - \D(x)/\wt{\D}(x))_{x \in [n]}.$ Observe that the pair of sequences $(1-\D/\wt{\D}, 1-\wt{\D}/\D)$ captures the relative error between the two distributions.
The sample complexity of the task of learning $\D$ in $\epsilon$-relative error is summarized by the following: 
\begin{theorem}
[Learning Phase]
\label{theorem:learning12}
For any $\epsilon, \delta > 0$, there exists an algorithm (\Cref{algo:learn}) that draws $N  = O\left(\frac{n}{\lambda(\vec{Q})\epsilon^2}\log(\frac{1}{\delta}) 
\right)$ samples from $\mathrm{Samp}(\Q;\D)$ satisfying Assumptions~\ref{assumption:identifiability} 
and~\ref{assumption:efficsample}, runs in time polynomial in $N$, and, with probability at least $1-\delta$, computes an estimate $\wt{\D}$ of the target distribution $\D$, that satisfies the relative error guarantee $ \max \big \{ \| 1- \D/\wt{\D} \|_{\infty}, \|1-\wt{\D}/\D\|_{\infty} \big \} \leq \epsilon$.
\end{theorem}
We execute \emph{once} the learning algorithm (\Cref{algo:learn}) that estimates $\D$ in relative error and we choose accuracy $\epsilon = 1/\sqrt{n}$.
We would like to shortly comment on the choice of the accuracy value: An accuracy of constant order would still potentially yield a biased random walk. After the re-weighting step, our
goal is to show that the target distribution becomes almost uniform and that the absolute spectral gap of the reformed matrix is of
order $\lambda(\vec Q)$ (see the proof of \Cref{lemma:tranform-properties}). In order to ensure this, any accuracy that vanishes with $n$ is sufficient. We chose to
set $\epsilon = 1/\sqrt{n}$ so that the complexity of learning $(\frac{n}{\lambda(\vec Q)\epsilon^2})$ matches that of the exact sampling procedure
$(\frac{n^2}{\lambda(\vec Q)})$. This learning step costs $n^2/\lambda(\vec Q)$ samples and the estimate $\wt{\D}$ is used as input to \Cref{algo:exact}. For the proof of \Cref{theorem:learning12}, we refer to the \Cref{sec:learning-phase}. We can now move to Line \ref{line:loop}.

\paragraph{Executing \Cref{algo:exact}.} We next focus on a single execution of the parameterized CFTP algorithm. Our goal is to transform the Markov chain of the Local Sampling Scheme defined in Equation~\eqref{equation:trans-matrix} to a modified Markov chain with an almost uniform stationary distribution. We are going to provide some intuition. Applying \Cref{theorem:learning12}, we can consider that, for any $x \in [n]$, there exists a coefficient $\wt{\D}(x) \approx \D(x)$. The idea is to use $\wt{\D}(x)$ and make the stationary distribution of the modified chain close to uniform. Intuitively, this transformation should speedup the convergence of the CFTP algorithm. The downscaling method of Bernoulli factories gives us a tool to do this modification.

We can implement the modified Markov chain via downscaling as follows: Consider an edge $\{x,y\}$ with transition probability pair $(p_{xy}, p_{yx})$. Without loss of generality, we assume that $\wt{\D}(y) > \wt{\D}(x)$ (which intuitively means that we should expect that $p_{xy} > p_{yx}$). Then, the downscaler leaves $p_{yx}$ unchanged and reduces the mass of $p_{xy}$ to make the two transitions almost balanced. 
Consider the LSS transition matrix $\vec M$ with $\vec M_{xy} = \Q(x, y) \frac{\D(y)}{\D(x)+\D(y)}$ and $\vec M_{xx} = 1 - \sum_{y \neq x} \vec M_{xy}$.  Also, let $\wt{\D}$ be an estimate for the distribution $\D$ (\Cref{theorem:learning12}). 
For the pair $(x,y)$, we modify the transition probability $p_{xy} := \vec M_{xy}$, only if $p_{xy} > p_{yx}$, to be equal to the following: $
\wt{p}_{xy} = \Q(x, y) \frac{\D(y)}{\D(x)+\D(y)} \frac{\wt{\D}(x)}{\wt{\D}(y)} \approx \Q(x, y) \frac{\D(x)}{\D(x)+\D(y)} = p_{yx}\,,$
where we use that $\wt{\D}(x) \approx \D(x)$ and $\wt{\D}(y) \approx \D(y)$. 
The transition probability from $x$ to $y$ corresponds to a Bernoulli variable $\mathrm{Be}(p_{xy})$, which is downscaled by $\wt{\D}(x)/\wt{\D}(y) < 1$. The modified transition probability $\wt{p}_{xy}$ can be implemented by drawing a $\Lambda \sim \mathrm{Be}(\wt{\D}(x)/\wt{\D}(y))$ and then drawing a $P \sim \mathrm{Be}(p_{xy})$ (from $\mathrm{Samp}(\Q; \D)$) and, finally, realizing the transition from $x$ to $y$ only if $\Lambda P = 1$. This implementation is valid since the two sources of randomness are independent. So, we perform a downscaled random walk based on the transition matrix $\wt{\vec M}$ with $\wt{\vec M}_{xy} = \wt{p}_{xy}$ (the ``small'' probability $p_{yx}$ remained intact).
We call $\wt{\vec M}$ the $\wt{\D}$-scaling of $\vec M$.
\Cref{lemma:tranform-properties12} summarizes the key properties of $\wt{\vec M}$.
The definition of the mixing time, used below, can be found at the notation \Cref{app:notation} and the full proof can be found at the \Cref{app:down}. 

\begin{lemma}
[Properties of Rescaled Random Walk]
\label{lemma:tranform-properties12}
Let $\D$ be a distribution on $[n]$ and consider an $\epsilon$-relative approximation $\wt{\D}$ of $\D$, as in \Cref{theorem:learning12} with $\epsilon = 1/\sqrt{n}$. Consider the transition matrix $\vec{M}$ of the Local Sampling Scheme $\mathrm{Samp}(\Q;\D)$ (see Equation~\eqref{equation:trans-matrix}), and let $\wt{\vec{M}}$ be the $\wt{\D}$-scaling of $\vec{M}$. Then, the transition matrix $\wt{\vec{M}}$ (i) has stationary distribution $\wt{\pi}_0$ with $\wt{\pi}_0(x) = \Theta(1/n)$ for all $x \in [n]$
and (ii) the mixing time of the transition matrix $\wt{\vec{M}}$ is $T_{\mathrm{mix}}(\wt{\vec{M}}) = O\left(\frac{\log(n)}{\lambda(\vec{Q})} \right)$.
\end{lemma}
\begin{proof}
[Proof Sketch]
Set $J_{xy} = \D(x) + \D(y).$
For a pair $(x,y) \in \calE$, before the downscaling, the pairwise $(x,y)$-comparison corresponds to the random coin $\left(\D(y)/J_{xy}, \D(x)/J_{xy}\right)$ and let $\D(y) > \D(x)$. After the downscaling, this coin becomes (locally) almost fair, i.e.,  $\left(\D(x)/J_{xy} + \xi_{xy}, \D(x)/J_{xy}\right).$ 
Hence, the modified transition matrix $\wt{\vec{M}}$ can be written as $
    \wt{\vec{M}} =\wt{  \vec D} \otimes \vec{Q} + \vec{Q} \circ [\xi_{xy}]\,,$
where 
$\wt{\vec D}$ is a symmetric matrix with $\wt{\vec D}_{xy} = \mathrm{min} \{\D(x), \D(y)\} / (\D(x) + \D(y)) = \wt{\vec D}_{yx}$,
$\vec Q$ is the Laplacian matrix of the graph $G_{\Q}$ and
$[\xi_{xy}]$ is a matrix with entries $\xi_{xy}$ where
$\xi_{xy}$ is only non-zero in the modified transitions $x \to y$, i.e., if $p_{xy} > p_{yx}$ then $\wt{\vec M}_{xy} = \Q(x,y) \cdot( \D(x)/J_{xy} +  \xi_{xy})$; for this pair, we set $\xi_{yx} = 0$. Finally, $\vec{Q} \circ [\xi_{xy}]$ denotes the modified Hadamard product\footnote{We have that $\Big( \vec{Q} \circ \vec A \Big)_{xy} = \Q_{xy} A_{xy}, \Big( \vec{Q} \circ \vec A \Big)_{xx} = -\sum_{y \neq x} \Q_{xy} A_{xy} .$} between $\vec{Q}$ and $[\xi_{xy}]$. Due to the accuracy $\epsilon$ of the learning algorithm, we have that $|\xi_{xy}| = O(1/\sqrt{n})$.

We analyze (i) the stationary distribution and (ii) the mixing time. For \textbf{Part (i)}, the chain $\wt{\vec M}$ remains irreducible and the detailed balance equations of $\wt{\vec M}$ satisfy
$
    \wt{p}_{xy}/ \wt{p}_{yx} = \D(y) \cdot \wt{\D}(x) / ( \wt{\D}(y) \cdot \D(x) ).
$
Since $\wt{\D}$ is an $\epsilon$-relative approximation of $\D$ with $\epsilon = O(1/\sqrt{n})$, it holds that for any $x \in [n]$:
$ \D(x) / \wt{\D}(x) \in [1-\epsilon, 1+\epsilon]$.
Hence, the unique stationary distribution $\wt{\pi}_0$ satisfies: $\wt{\pi}_0(x)/\wt{\pi}_0(y) \in [1-\Theta(1/\sqrt{n}), 1+\Theta(1/\sqrt{n})]$ for any $x,y \in [n]$. So, it must be the case that $\wt{\pi}(x) = \Theta(1/n)$. 

For \textbf{Part (ii)}, in order to show that the mixing time is of order $O(\log(n)/\lambda(\vec Q))$, it suffices to show (by \Cref{lemma:mixing-time-peres}) that (i) the minimum value of the stationary distribution is $O(1/n)$ and (ii) the absolute spectral gap $\Gamma(\wt{\vec M})$ of $\wt{\vec M}$ satisfies  $\Gamma(\wt{\vec M}) = \Omega(\lambda(\vec{Q}))$, where $\lambda(\vec{Q})$ is the minimum non-zero eigenvalue of the Laplacian matrix $\vec{Q}$. The first property follows directly from the closeness of the stationary distribution of $\wt{\vec M}$ to the uniform one. We now focus on the latter.
Our goal is to control the absolute spectral gap $\Gamma(\wt{\vec M})$. In fact, this matrix can be seen as a perturbed version of the matrix $\wt{  \vec D} \otimes \vec{Q}$ as we discussed in the beginning of the proof sketch. The perturbation noise matrix $\vec E = \vec Q \circ [\xi_{xy}]$ is a matrix whose entries contain (among others) the approximation errors $[\xi_{xy}]$. Hence, we can use Weyl's inequality \cite{tao2012topics} in order to control the absolute spectral gap of $\wt{\vec M}$. 
This inequality
gives a crucial perturbation bound for the singular values for general matrices. 
Since the matrix $\wt{\vec M}$ can be seen as a perturbation (as we discussed previously), we can (informally speaking) use Weyl's inequality to 
control the singular values. 
After some algebraic manipulation, we derive that $\Gamma(\wt{\vec M}) = \Omega(\lambda(\vec Q))$.
\end{proof}
So, a \emph{single iteration} of the parameterized CFTP algorithm with transition matrix $\wt{\vec M}$ (see Line \color{black} \ref{line:loop} \color{black} of \Cref{algo:total-exact}) guarantees that with an expected number of $N = \wt{O}(n T_{\mathrm{mix}}(\wt{\vec M})) = \wt{O}\left(n \log(n)/\lambda(\vec Q) \right)$ samples, it reaches Line \ref{line:ber} with a sample $y = F_t(1)$ that is generated by the parameterized CFTP algorithm (the iterations in Line \ref{line:while}). We then have to study the \emph{quality of the output sample $y$}. Crucially, by the utility of the standard CFTP algorithm, the law of $y$ is the normalized measure that is proportional to $\D/\wt{\D}$, i.e., $\wt{\pi}_0(y) = (\D(y)/\wt{\D}(y))/\sum_{x \in [n]} \D(x)/\wt{\D}(x)$ (this is the stationary distribution of the modified random walk). Hence, we cannot simply output $y.$ For this reason, we perform rejection sampling. Specifically, as we will see, a single execution of \Cref{algo:exact} either produces a single sample distributed as in $\D$ with some known probability $p$ or fails to output any sample with probability $1-p$ (this is because we do not sample from $\D$ but from an almost uniform distribution). If the rejection sampling process is successful, then the while loop (Line \ref{line:while-2}) of \Cref{algo:total-exact} will break and the clean sample will be given as output. One can show that $\Theta(n)$ iterations of the parameterized CFTP algorithm suffice to generate a perfect sample. Hence, the total sample complexity is of order $O(n^2 \log(n) \log(n/\delta) / \lambda(\vec Q))$ (at each step $\delta' = \Theta(\delta/n))$. We continue with a sketch behind the rejection sampling process of Line \ref{line:reject} of \Cref{algo:exact}.
\begin{claim}
[Sample Quality and Rejection Sampling]
\Cref{algo:exact}, with input the output of the Learning Phase, outputs a state $x \in [n]$ with probability proportional to $\D(x)$ or outputs $\perp$. Moreover, \Cref{algo:total-exact} outputs a perfect sample from $\D$ with probability 1.
\end{claim}
To see this, set $A =  \sum_{y \in [n]} \D(y)/\wt{D}(y)$.
At the end of Line \ref{line:end}, \Cref{algo:exact} outputs $x \propto \D(x)/\wt{D}(x)$, since the unique stationary distribution is $\wt{\pi}(x) = (\D(x)/\wt{\D}(x)) / A$ and the utility of the CFTP guarantees that
the outputs state follows the stationary distribution. For the sample $x$, we perform a rejection sampling process, with acceptance probability $\wt{\D}(x)$.
\Cref{algo:exact} has $n+1$ potential outputs: It either prints $x \in [n]$ or $\perp$ indicating failure. The arbitrary sample $x \in [n]$ is observed with probability $\wt{\pi}(x) \cdot \wt{\D}(x) = \D(x)/A$. The remaining probability mass is assigned to $\perp.$ We claim that the output of Line \ref{line:out} of \Cref{algo:total-exact} has law $\D$. Observe that the whole stochastic process of \Cref{algo:total-exact} outputs $x \in [n]$ with probability
$\sum_{i = 0}^{\infty} \Pr[\mathrm{Reject}]^i \Pr[x \textnormal{ is Accepted}]
= \D(x)/ A \cdot \sum_{i = 0}^{\infty} \Pr[\mathrm{Reject}]^i$, where
$
\Pr[\mathrm{Reject}] = \frac{1}{A} \sum_{y \in [n]} \frac{\D(y)}{\wt{\D}(y)} (1 - \wt{\D}(y)) = 1 - \frac{1}{A} \in (0,1)\,.$ As a result, we get that the probability that $x$ is the output of the above stochastic process is $\D(x).$

\section{Conclusion} 
We close with an open question: Is the bound $\wt{O}(n^2/\lambda(\vec Q))$ tight? The main difficulty is that it is not clear
how to obtain a lower bound against \emph{any possible} exact sampling algorithm. 
We would like to mention that the dependence on $1/\lambda(\vec Q)$ is expected and intuitively unavoidable. This term is connected to the mixing time. Let us become more specific. The (sample) complexity of our proposed method is $\wt{O}(n^2/\lambda(\vec Q))$. The term $1/\lambda(\vec Q)$ appears due to the mixing time of the (transformed) Markov chain. The complexity of CFTP is $\wt{\Theta}(n T_\mathrm{mix})$. Hence, any CFTP-based exact sampler (or more broadly random walk based algorithm) will have that kind of spectral dependence.

Similarly, a dependence on the support’s size $n$ is necessary in general. Note that the structure of the target distribution $\D$ is crucial: if the target is uniform (or generally “flat”), then a bound of order $\wt{O}(n/\lambda(\vec Q))$ is obtained by the naive CFTP method; for distributions with modes, our algorithm achieves quadratic dependence on $n$. The sample complexity of a potential (not necessarily random-walk based) exact sampling algorithm will depend on $\Q$, but might achieve a sample complexity not directly dependent on $\lambda(\vec Q)$. The unclear part is our algorithm’s dependence on $n$. While our algorithm attains a bound of $\wt{O}(n^2)$, it is not obvious whether an algorithm with sample complexity $o(n^2)$ exists. 

Similarly, it is interesting to ask whether there are (efficient) ways to transform the chain in order to get a faster sample corrector. Our learning method finds a vector $\vec p$ (input of \Cref{algo:exact}) that makes the re-weighted distribution close to uniform. Is it possible to find another $\vec p$ (without our learning algorithm) such that (i) when transforming the chain with $\vec p$, we obtain a fast mixing chain; and (ii) the success probability in the rejection sampling step will be sufficiently high?

This work provided a theoretical understanding to perfect sampling from pairwise comparisons. We believe it is an interesting direction to experimentally evaluate our proposed methodology.

\paragraph{Acknowledgments.} We would like to thank the anonymous reviewers of NeurIPS 2022 for their feedback and comments.

\bibliography{refs}

\appendix

\section{Preliminaries}
In this section, we provide some preliminaries and some useful tools about (i) concentration of random matrices, (ii) random walks and (iii) CFTP. The reader may skip this section.

\subsubsection{Random Matrices}
\label{sec:concentration}
We continue with some definitions for random matrices, needed for the proof of \Cref{lemma:concetration-fiedler}. The following can be found at \cite{tropp2015introduction}.

\begin{definition}
Let $(\Omega, \mathcal{F}, \mu)$ be a probability space. A random matrix $\vec X$ is a measurable map $\vec X : \Omega \rightarrow \mathcal{M}^{n_1 \times n_2}$.
\end{definition}
The entries of $\vec X$ may be considered complex random variables that may or may not be correlated with each other. A finite sequence $\{\vec X_i\}$ of random matrices is independent when
\[
\mu( \vec X_k \in F_k \text{ for each } k) = \prod_k \mu(\vec X_k \in F_k)\,,
\]
for any collection $\{F_k\}$ of Borel subsets of $\mathcal{M}^{n_1 \times n_2}$.
\begin{proposition}
[Hermitian Matrix Chernoff Bounds (see \cite{tropp2015introduction})]
\label{prop:conc}
Consider a finite sequence $\{\vec X_i\}$ of $m$ independent, random, Hermitian square matrices with common dimension $n$. Assume that $0 \leq \lambda_{\min}(\vec X_i) \leq \lambda_{\max}(\vec X_i) \leq M$ for any $i \in [m]$. Let $\vec Y = \sum_i X_i$. Then, for $\epsilon \in [0,1)$:
\begin{equation*}
    \Pr\Big [\lambda_{\min}(\vec Y) \leq (1-\epsilon)\lambda_{\min}(\E \vec Y) \Big ] \leq n(e^{-\epsilon}/(1-\epsilon)^{1-\epsilon})^{\lambda_{\min}(\E \vec Y)/ M} \leq n \exp \Big (- \epsilon^2 \lambda_{\min}(\E \vec Y) /(2M)\Big )\,.
\end{equation*}
\end{proposition}

\subsubsection{Random Walks and Markov Chains}
\label{appendix:prelims}
This section is mostly based on \cite{levin2017markov} and we refer the reader to \cite{levin2017markov} for a thorough exposition.

\paragraph{Markov Chains.} Let $\Omega$ be a finite state space. A Markov chain is a sequence of random variables $X_0, X1, \dots$ that satisfy the Markov property, i.e.,
\begin{equation*}
    \Pr[X_{t+1} = x | X_0 = x_0, \dots, X_t = x_t] = \Pr[X_{t+1} = x | X_t = x_t]\,.
\end{equation*}
A Markov chain is called time-homogeneous, if the RHS of the above equation does not depend of $t$. Such a chain is associated with a transition matrix $\vec{P} = \{\vec{P}(x,y)\}$, where $(x,y) \in \Omega \times \Omega$. It holds that
\begin{equation*}
    \vec{P}(x,y) = \Pr[X_{t+1} = y | X_t = x] \text{ for all } x,y \in \Omega, t \in \mathbb{N}\,.
\end{equation*}
A Markov chain is \emph{ergodic} if there exists a time $t \in \mathbb{N}$ such that $\vec{P}^{(t)}(x,y) > 0$ for any $x,y \in \Omega$, i.e., there exists a finite time $t$ so that the probability of going from any vertex to any other in $t$ steps is positive. For finite state Markov chains, ergodicity is equivalent to irreducibility and aperiodicity. A Markov chain is \emph{irreducible} if for any two states $x,y \in \Omega$, there exists a time step $t$ such that $\vec{P}^{(t)}(x,y) > 0$ and is \emph{aperiodic} if, for any state $x$, it holds that $\mathrm{gcd}\{t | \vec{P}^{(t)}(x,x) > 0\} = 1$. Observe that a Markov chain with self-loops is aperiodic.

\paragraph{Stationary Distribution.} A stationary distribution $\vec \pi \in \Delta^n$, i.e., the $(n-1)$-dimensional simplex whose vertices are the $n$ standard unit vectors, is defined as the fixed point of the transition matrix $\vec{P}$, that is $\vec \pi^T \vec{P} = \vec \pi^{T}$. An ergodic Markov chain has a \emph{unique stationary distribution} $\vec \pi$ and, as $t$ increases, it converges to it. An interesting property of Markov chains is the reversibility. A Markov chain is \emph{reversible} if there exists a distribution $\vec \pi$ that satisfies the detailed balanced equations:
\begin{equation}
    \vec \pi(x) \vec{P}(x,y) = \vec \pi(y) \vec{P}(y,x) \text{ for all } x,y \in \Omega\,.
\end{equation}
In this case, we can verify that $\vec \pi$ is a stationary distribution. We can simply write $\pi$ for the stationary distribution with associated probability vector $\vec \pi$. 

\paragraph{Mixing Time.} It is important to understand the convergence time of a Markov chain to its stationary distribution $\pi$. 
A crucial random variable for convergence time is the \emph{mixing time}.
\begin{definition}[Mixing Time]
Let $0 < \epsilon < 1/2$. Let $M$ be an ergodic Markov chain on a finite state space $\Omega$ with stationary distribution $\pi$. Then, the mixing time with accuracy $\epsilon$ of $M$ equals:
\begin{equation*}
    T_{\mathrm{mix}}(\vec P; \epsilon) = \inf \{ t > 0 : \max_{x \in \Omega} \mathrm{TV}( \vec{P}_x(X_t) , \pi) \leq \epsilon \}\,,
\end{equation*}
where $\vec{P}_x(X_t)$ is the distribution of the state $X_t$ at time $t$ for starting state $x \in \Omega$.
\end{definition}
For a reversible transition matrix $\vec{P}$, let its spectrum be:
\begin{equation*}
    1 = \lambda_1 \geq \lambda_2 \geq \dots \geq \lambda_n \geq -1\,.
\end{equation*}
Note that $\lambda_2 < 1$ if and only if the chain is irreducible (exactly one connected component) and $\lambda_n > -1$ if and only if the chain is aperiodic (e.g., not a bipartite graph). We define the \emph{absolute spectral gap} of $\vec{P}$ to be the difference: $\Gamma(\vec{P}) = 1 - \max\{ |\lambda_2(\vec{P})|, |\lambda_n (\vec{P})| \}$. It holds that if $\vec{P}$ is aperiodic and irreducible, then $\Gamma(\vec{P})$ is strictly positive. One could also define the spectral gap $\lambda(\vec P) = 1-\lambda_2(\vec{P})$. The two gaps are equal when the chain is lazy, i.e., when for any state $x \in \Omega,$ it holds that $\vec{P}(x,x) \geq 1/2$.

\begin{lemma}
[Bounding $T_{\mathrm{mix}}$, see~\cite{levin2017markov}]
\label{lemma:mixing-time-peres}
Let $0 < \epsilon < 1$. Assume that the transition matrix $\vec P$ is aperiodic, irreducible and reversible with respect to $\pi$. Then, it holds that:
\begin{equation*}
    (t_{\mathrm{rel}} - 1) \log \Big (\frac{1}{2\epsilon} \Big ) \leq T_{\mathrm{mix}}(\vec P; \epsilon) \leq \log \Big (\frac{1}{\epsilon\pi_{\mathrm{min}}} \Big ) t_{\mathrm{rel}}\,,
\end{equation*}
where $t_{\mathrm{rel}} = 1/\Gamma(\vec P)$ is the relaxation time of the Markov chain, i.e., the inverse absolute spectral gap.  
\end{lemma}
\paragraph{Coalescencing Random Walks.} Let $\Omega = [n]$. In a coalescing random walk, a set of $n$ particles perform independent discrete time random walks on an undirected connected graph $\G = (V,E)$ with $|V| = n$, with each particle initially placed at a single (distinct) vertex $x \in V$. In each time step, all particles move simultaneously. Whenever two or more
particles meet at a vertex, they unite to form a single particle which
then continues to make a random walk through the graph. The \emph{coalescence time} $T_{\mathrm{coal}}$ is a random variable and is the first time when all particles coalesce. More formally, we can define \emph{coalescence} as follows.
\begin{definition}
[Coalescence of Stochastic processes (see \cite{huber2016perfect})]
Let $\mathcal{A}$ be a collection of stochastic processes, defined over a common index set $\mathcal{I}$ and common state space $\Omega$. If there is an index $i \in \mathcal{I}$ and state $x \in \Omega$ such that, for all stochastic processes $X \in \mathcal{A}$, it holds that $X_i = x$, then we say that the stochastic processes have coalesced.
\end{definition}

\paragraph{Coalescence Time of Random Walks.} Consider the random variable $X^{(i)}_t$, that indicates the position of the $i$-th particle at time $t$. The coalescence time is equal to:
\begin{equation*}
    T_{\mathrm{coal}} = \inf_{t>0} \Big \{ X_{t}^{(i)} = X_t^{(j)} \text{ for any } i \neq j, i,j \in [n] \Big  \}\,.
\end{equation*}

We conclude this section with a folklore result, that deals with the possible locations in the complex plane of the eigenvalues of a square matrix $\vec{A} \in \mathcal{M}_{n}$.
\begin{lemma}
[Geršgorin's Theorem (see~\cite{horn2012matrix})]
\label{lemma:gersgorin}
Let $\vec{ A} = [\vec{A}_{ij}] \in  \mathcal{M}_n$. For any $i \in [n]$, let
\[
R_i(\vec{A}) = \sum_{j \neq i}|\vec{A}_{ij}|\,,
\]
and consider the $n$ Geršgorin disks
\[
B_i = \{ z \in \mathbb{C} : |z - \vec{ A}_{ii}| \leq R_i(\vec{A}) \}\,.
\]
Then, the eigenvalues of $\vec{A}$ are in the union of the Geršgorin disks, i.e.,
\[
\mathrm{spec}(\vec{A}) \subseteq \bigcup_{i=1}^{n} B_i\,.
\]
\end{lemma}


\subsubsection{Exact Sampling and CFTP Preliminaries}
\label{appendix:exact-sampling-prelims}

Markov chain Monte Carlo (MCMC) methods constitute a class of algorithms for sampling from probability measures and arise naturally in various fields of science such as theoretical computer science (e.g., approximation algorithms for \#$\mathrm{P}$-complete problems (see~\cite{jerrum2003counting,jerrum2004polynomial})), statistical physics (e.g., in order to understand phase transition phenomena for Ising models) and statistics. However, the theory of MCMC, in terms of exact sampling, can be seen as an asymptotic analysis (i.e., in the limit, the total variation distance vanishes). Perfect simulation is analogous to MCMC and deals with techniques for designing algorithms that return exact draws from the target distribution, instead of long-time approximations. Exact sampling comprises a well-studied field (see~\cite{kendall2005notes, huber2016perfect}) and has numerous applications in computer science (e.g., in approximate counting~\cite{huber1998exact}). The importance of exact simulation gave rise to various procedures in order to generate perfect samples. See \cite{propp1996exact,green1999exact,wilson2000couple,haggstrom2000propp,fill2000randomness,huber2004perfect} for a small sample of this line of research. 
\begin{figure}[ht]
    \centering
    \includegraphics[scale=0.4]{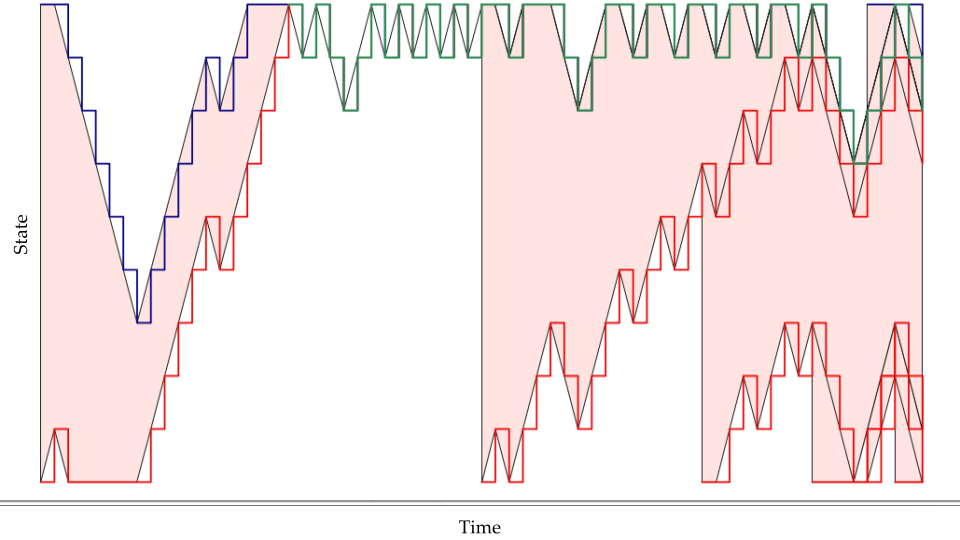}
    \caption{Execution of the CFTP Algorithm for a random walk on the path graph. 
    Observe that it suffices to execute the random procedure only for the two extreme points, since their trajectories dominate any intermediate point of the path. 
    After a series of executions which did not coalesce (see the right part of \Cref{fig:cftp-exec1}), we observe that the two stochastic evolutions, the blue and the red one in the left part of \Cref{fig:cftp-exec1}, coalesce and continue as a single trajectory until $t = 0$ (see the green trajectory).
    A sample execution code can be found at~\url{https://warwick.ac.uk/fac/sci/statistics/staff/academic-research/kendall/personal/perfect_programs/}}
    \label{fig:cftp-exec1}
\end{figure}

Coupling from the past (CFTP) is a technique developed by Propp and Wilson \cite{propp1996exact, propp1998get}, that provides an exact random sample from a Markov chain, that has law the (unique) stationary distribution. The algorithm assumes that a particle has been running on the Markov chain since time $-\infty$ (arbitrarily long in the past) and we are concerned with the location of the particle at (the fixed) time $t = 0.$ The fact that the stopping time is deterministic and not a random variable (as in MCMC) is crucial. Since the particle performs a random walk in the Markov chain infinitely long, one would intuitively believe that the particle at time $t = 0$ is distributed according to the stationary distribution.

\begin{algorithm}[ht] 
\caption{Coupling From The Past Algorithm}\label{algo:cftp}
\begin{algorithmic}[1]
\Procedure{\textsc{CouplingFromThePast}()}{}\Comment{\emph{Assuming (adaptive) access to the Markov oracle}.}
\State $t \gets 0$
\State $F_{(t,0]}(i) \gets i, \text{ for any } i \in [n]$
\State \textbf{while}~{$F_{(t,0]}$ has not coalesced}~\textbf{do}
\State~~~~ $t \gets t-1$
\State~~~~ $f_t \gets \mathrm{Markov}()$
\State~~~~ $F_{(t,0]} \gets F_{(t+1,0]} \circ f_t$
\State \textbf{end}
\State \textbf{return} $F_{(t,0]}(1)$   
\EndProcedure
\end{algorithmic}
\end{algorithm}

We define $F_{(t,0]}$ as the $t$-step evolution of the Markov chain from time $t \in \mathbb{Z}_{\leq 0}$ to the fixed time $T = 0$ (evolution \emph{from the past}). Note that we can decompose the evolution of the chain into $t$ independent applications of the random function $f$, which encodes the information of the random walk, i.e., $\Pr[f(x) = y] = \vec P_{xy}$ for any pair of states $x,y \in \Omega$. Hence, we have that
\[
F_{(t,0]} = f_{-1} \circ f_{-2} \circ \ldots \circ f_{t}\,.
\]

We can think of the data structure $F_{(t,0]}$ as a list of $n$ values, each one capturing the evolution of the Markov chain, initiated in a different state at time $t$ and stopped at time $0$. Similarly, the oracle call $\mathrm{Markov}()$ corresponds to $n$ adaptive calls, each one giving a transition from a fixed state $x \in \Omega$, where the next state is distributed according to $\vec P(x, \cdot)$. Hence, $f_t$ is an $n$-dimensional vector with values in $\Omega$. We refer the reader to Figure \ref{fig:cftp-exec1} for an execution of the CFTP algorithm.

The following fact
is the main result behind CFTP. 
\begin{fact}[\cite{propp1996exact}]
    With probability $1$, the Coupling from the Past algorithm returns a value, and this value
is distributed according to the stationary distribution of the Markov chain.
\end{fact}

The following proposition provides an upper bound on the convergence of the CFTP algorithm.
\begin{proposition}
[\cite{zahavy2019average}]
\label{proposition:cftp-convergence}
Let $\pi$ be the stationary distribution of an
ergodic Markov chain $\vec P$ with $n$ states. We run $n$ simulations of the chain each starting at a different state. When two or more simulations coalesce, we merge them into a single simulation. With probability at least $1 - \delta$, all $n$
chains have merged after at most $O(n T_{\mathrm{mix}}(\vec P; 1/4) \log(1/\delta))$ iterations.
\end{proposition}

\section{The Proof of \Cref{theorem:naive} (Direct Exact Sampling from LSS)}
\label{appendix:proof-naive}
In this section, we provide the proof of \Cref{theorem:naive}. Let us begin with the algorithm. We remark that in the proof we will use the more technical notation $F_{(a,b]}$ to denote the evolution of the chain from $a$ to $b$. In \Cref{algo:exact-naive}, this is abbreviated as $F_t$ to denote $F_{(t,0]}$.

\begin{proof}
 At each iteration, the algorithm draws a single sample from $\mathrm{Samp}(\Q;\D)$ 
Then, the algorithm performs the following procedure for each state $v \in [n]$: let the drawn sample be $(i,j,w),$ where $(i,j) \in \calE$ and $w \in \{0,1\}$ indicating the winning node between $i$ and $j$, i.e., the node $i$ stays in $i$ (if $w = 1$) or node $i$ moves to $j$ (if $w = 0$). For any $v \in [n]$, the algorithm checks whether $v \in \{i,j\}$. If not, the state $v$ remains at $v$; otherwise, assume that $i = v$. Then, the transition $v \to j$ is implemented only if $w = 0$, otherwise it remains at $v$. We claim that this process simulates the transitions of the matrix $\vec M$ and has unique stationary distribution $\D$.  
The probability that the transition $v \to j$ occurs is
\begin{align*}
\Pr[v \to j] &= \Pr[ (v,j,w) \text{ is drawn and } w = 0] = \Pr[(v,j) \text{ is drawn}] \cdot \Pr[w = 0 | (v,j) \text{ is drawn}] \\
&= \Q(v,j) \cdot \frac{\D(j)}{\D(v) + \D(j)} = \vec M_{v j}\,.
\end{align*}
Hence, for $v \in [n]$ and by the above process, the transition probabilities are given by the matrix $\vec M$, that is a transition matrix with unique stationary distribution $\D$. We will see that this random process gives a perfect sample.

Consider the state space $\Omega = [n]$ and let $T$ be a positive integer. Let $(X_t)_{t \in \mathbb{Z}}$ be a random walk on $[n]$ with transition probabilities as in $\vec M$ (see Equation~\eqref{equation:trans-matrix}).
Let $F_{(a,b]} : [n] \rightarrow [n]$ be such that $F_{(-t_0,T]}(x_0)$ denotes the state of the walk at time $t = T$, i.e., $X_{T} \in [n]$, when the random walk begun at time $t = -t_0$ (i.e., \emph{in the past}) with initial state $X_{-t_0} = x_0$. Specifically, for each $t_0, T$, we have that $F_{(-t_0, T]} : [n] \rightarrow [n]$ is a random map and
$F_{(-t_0, T]}(x_0)$ is a realization of the Markov chain, with initial state $x_0$, executed from time $-t_0$ to time $T$. We call $F_{(-t_0, T]}$ a \emph{stochastic flow}.

The CFTP method is initialized at time $t = -t_0$ into the past and places $n$ particles (i.e., simulates $n$ concurrent random walks), one at each state $x \in [n]$. CFTP terminates when the stochastic flow $F_{(-t_0, 0]}$ is constant (it has a single point random image, i.e., $F_{(-t_0, 0]}(x) = F_{(-t_0, 0]}(y)$, for all $x, y$). Namely, CFTP terminates when all the random walks that start at time $t = -t_0$ have coalesced at time $t = 0$ at a single $z \in [n]$. Then, the algorithm outputs the (common) state $z = F_{(-t_0,0]}(x)$, for any $x \in [n]$. 

Since the transitions of each random walk are performed with respect to $\vec M$,
the resulting sample $z$ is distributed according to the stationary distribution $\D$. This can be shown as follows (see also~\cite{propp1996exact, huber2016perfect}): Let $\mathcal{L}(X)$ denote the distribution of the random variable $X$. Since the chain is aperiodic and irreducible, coalescence occurs almost surely. Specifically, assume that coalescence occurs when the simulation starts at a finite time $t = -t_0$. First, observe that if coalescence occurs at time $-t_0$, then if CFTP starts at any time $-t < -t_0$, we end up in the same state (because the randomness in $[-t_0, 0]$ remains the same), i.e., $F_{(-t,0]}(x) = F_{(-t_0, 0]}(x)$ for any $-t < -t_0$ and any state $x \in [n]$. Since, coalescence has occurred, we can work with a fixed starting state $x_0$ and let $F_{(-t,0]} = F_{(-t,0]}(x_0)$. Hence, it holds that  $\mathcal{L}(F_{(-t,0]}) = \mathcal{L}(F_{(-t_0, 0]})$ and so 
\[
\mathrm{TV}(\mathcal{L}(F_{(-t_0, 0]}), \D) = \lim_{t \rightarrow \infty} \mathrm{TV}(\mathcal{L}(F_{(-t, 0]}), \D)\,.
\]

We note that the distribution of the time-backward coalescing random walk is equal to the distribution of the time-forward random walk, i.e., $\mathcal{L}(F_{(-t, 0]}) = \mathcal{L}(F_{(0, t]})$. Thus, $\mathrm{TV}(\mathcal{L}(F_{(-t_0, 0]}), \D) = \lim_{t \rightarrow \infty} \mathrm{TV}(\mathcal{L}(F_{(0, t]}), \D) = 0$. So, the output sample of the CFTP algorithm is distributed according to $\D$. Using \Cref{proposition:cftp-convergence}, we get that, with high probability, coalescence occurs after $O(n T_{\mathrm{mix}}(\vec M; 1/4))$ iterations. From \Cref{lemma:mixing-time-peres} with $\epsilon = 1/4$, we get that the mixing time of the Markov chain is 
$O(\log(1/\D_{\mathrm{min}})/\Gamma(\vec M))$.
Finally, note that in order to perform a single iteration of the CFTP method, it requires to draw a single sample from the Local Sampling Scheme oracle. This concludes the analysis of the sample complexity. The time complexity is also polynomial in the number of samples \cite{propp1996exact}.
\end{proof}

\begin{remark}
\label{remark:sample-complexity-naive}
As we mentioned in the discussion after \Cref{thm:inf:naive}, in order to obtain \Cref{theorem:naive} only \Cref{assumption:identifiability} 
is required. Under only this assumption, a quite large amount of Local Sampling Schemes is allowed and the resulting sample complexity is $\wt{O} \left(\frac{n \log (1/\D_{\mathrm{min}})}{\Gamma( \vec{M}) } \right)$. However, if we additionally assume that \Cref{assumption:efficsample} holds, as we did in the statement of \Cref{thm:inf:naive}, we accept only LSSs whose target distribution satisfies $\D_{\mathrm{min}} = \Omega(2^{-n}).$ This yields the (non-tight) bound $\wt{O}(n^2/\Gamma(\vec M))$ of \Cref{thm:inf:naive}. We preferred to omit this detail in \Cref{thm:inf:naive} and emphasize only on the exponential dependence due to $\Gamma(\vec M)$ for many natural instances. 
\end{remark}

\section{The Proof of \Cref{thm:sampling} (Exact Sampling using Learning)}
\label{appendix:proof-main}
The main result of this section is the convergence analysis of the Coupling From the Past algorithm when the algorithm performs the rescaling transformation internally in each transition.
\begin{theorem}
[Convergence of CFTP]
\label{theorem:convergence-cftp}
Assume that \Cref{assumption:identifiability} and \Cref{assumption:efficsample} both hold.
\Cref{algo:exact}, assuming access to a Local Sampling Scheme $\mathrm{Samp}(\Q; \D)$ with 
weighted Laplacian matrix $\vec Q$ (see Equation \eqref{eq:laplacian-q}) associated with the pair distribution $\Q$ over $[n] \times [n]$, satisfies the following:
\begin{itemize}
\item[$(i)$] \Cref{algo:exact} reaches {Line \color{black}\ref{line:ber}\color{black}} with probability $1$.
\item[$(ii)$] For any $\delta \in (0,1)$, a single execution of of \Cref{algo:exact} uses 
$\wt{O} \left(
\frac{n \log(n) \log(1/\delta)}{\lambda(\vec{Q})} 
\right)$
samples with probability $1-\delta$ and
the running time is polynomial in the number of samples.
\item[$(iii)$] \Cref{algo:exact}, with input the output of the Learning Phase, outputs a state $x \in [n]$ with probability proportional to $\D(x)$ or outputs $\perp$. Moreover, \Cref{algo:total-exact} outputs a perfect sample from $\D$.
\end{itemize}
Using the above properties, for any $\delta > 0$, \Cref{algo:total-exact} draws, with probability at least $1-\delta$, $N = {O}\left( n^2 \log^2(n)/\lambda(\vec Q) \cdot \log(1/\delta) \right)$
samples from a Local Sampling Scheme $\mathrm{Samp}(\Q;\D)$, runs in time polynomial in $N$, and outputs a sample distributed as in $\D$. 
\end{theorem}

\begin{proof} [Proof \emph{of}~\Cref{theorem:convergence-cftp}]
Consider a target distribution $\D$, supported on $[n]$ and sample access to a Local Sampling Scheme $\mathrm{Samp}(\Q;\D)$ for a pair distribution $\Q$. Let $\vec{M} = \vec Q \otimes \vec D$ be the transition matrix of the Markov chain associated with the Local Sampling Scheme graph. 
From \Cref{theorem:learning12} with accuracy $\epsilon > 0$, there exists an algorithm that, with high probability, computes the rescaling of the Markov chain that corresponds to $\vec{M}$, which we denote by $\wt{\vec{M}}$, the transition matrix of the rescaled (by $1/\wt{\D}$) Markov chain. The learning algorithm uses $\wt{O}(n^2/(\lambda(\vec Q)))$ with probability $1-\delta$ and is executed only once. We now analyze a single execution of \Cref{algo:exact}.
\begin{claim}
[Termination of CFTP]
\Cref{algo:exact} reaches {Line \color{black}\ref{line:ber}\color{black}} with probability $1$.
\end{claim}
\begin{proof}
Since the chain $\wt{\vec{M}}$ is ergodic, there is a $T$ such that for any pair of states $(x,y)$, the probability $\wt{\vec{M}}^T(x,y) > 0$. Similarly to the proof sketch of the naive exact sampling algorithm (see also \Cref{theorem:naive}), let $(X_t)_{t \in \mathbb{Z}}$ be a random walk on $[n]$ and define the mapping $F_{(a,b]} : [n] \rightarrow [n]$, so that $F_{(-t_0,T]}(x_0)$ is the state of the random walk at time $t = T$, i.e., $X_{T} \in [n]$, where the random walk begun at time $t = -t_0$ (in the past) with starting state $X_{-t_0} = x_0$.
Then, the $n$ random walks have coalesced after $T$ steps if and only if $|\mathrm{Image}(F_{(-T,0]})| = 1$. By the ergodicity of the chain, each one of the events $E_T = \{ |\mathrm{Image}(F_{(-T, 0]})| = 1\}$ for $T \in \mathbb{Z}_{\geq 0}$ has strictly positive probability to occur. Since the events are independent, with probability $1$, there will be an event $E_{T^{\star}}$ that occurs. Then, for any $T>T^{\star},$ the desired property holds.  
\end{proof}

\begin{claim}
[Sample Complexity of a Single Iteration]
It holds that each execution of \Cref{algo:exact}
uses $O \left (\frac{n \log(n) \log(1/\delta)}{\lambda(\vec Q)} \right)$ samples from the Local Sampling Scheme with probability $1-\delta$.
\end{claim}
\begin{proof}
The CFTP draws a single sample for each iteration and performs a step according to the matrix $\wt{\vec M}$ due to the rescaling. First, the unique stationary distribution is almost uniform (\Cref{lemma:tranform-properties}) and so $\log(1/\D_{\mathrm{min}}) = \Theta(\log(n))$. This implies, by \Cref{lemma:mixing-time-peres} and \Cref{lemma:tranform-properties}, that the mixing time of $\wt{\vec M}$ for some constant accuracy $\epsilon_0 = 1/4$ is $T_{\mathrm{mix}}(\wt{\vec M}; 1/4) = O(\log(n)/\lambda(\vec Q))$. Hence, by \Cref{proposition:cftp-convergence}, we get the desired result. 
\end{proof}
Moreover, the update time of each execution of the CFTP algorithm is polynomial in $n$. Note that, since during the Learning Phase, the algorithm uses $O(n^2/\lambda(\vec Q) )$ and is executed once, the sample complexity of \Cref{algo:total-exact} follows, since \Cref{algo:exact} will be executed $\Theta(n)$ times.

\begin{claim}
[Rejection Sampling]
\Cref{algo:exact}, with input the output of the Learning Phase, outputs a state $x \in [n]$ with probability proportional to $\D(x)$ or outputs $\perp$. Moreover, \Cref{algo:total-exact} outputs a perfect sample from $\D$.
\end{claim}
\begin{proof}
Let us set $A =  \sum_{y \in [n]} \D(y)/\wt{D}(y)$.
At the end of Line \ref{line:end}, the parameterized CFTP algorithm (\Cref{algo:exact}) outputs a sample $x \propto \D(x)/\wt{D}(x)$. This is because the unique stationary distribution is $\wt{\pi}(x) = (\D(x)/\wt{\D}(x)) / A$ and the utility of the CFTP mechanism guarantees that
when the Algorithm initiates a CFTP simulation from time $T = -T^{\star}$, the generated sample $x = F_{(-T^{\star},0]} =  F_{(-\infty,0]}$ will be distributed according to $\wt{\pi}$. This follows from the observation that the distribution of the state $F_{(-\infty,0]}$ is equal to the distribution of the state $F_{[0,+\infty)}$ obtained by running the simulation up to the limit $T \rightarrow +\infty$ forward in time. Hence, we have that $x$ has law the unique stationary distribution $\wt{\pi}$.
We now have to remove the bias induced by the learning step. For the sample $x$, we perform a rejection sampling process, with acceptance probability $\wt{\D}(x)$.
The \Cref{algo:exact} has $n+1$ potential outputs: It either prints $x \in [n]$ or $\perp$ indicating failure. The sample $x$ is the output of the algorithm with probability $\frac{\D(x)/\wt{\D}(x)}{A} \cdot \wt{\D}(x) = \D(x)/A$. This holds for any $x \in [n]$. The remaining probability mass is assigned to $\perp.$ Hence, a single execution of \Cref{algo:exact} outputs a point $x \in [n]$ with probability $\D(x)/A$ and outputs ``reject'' with probability $1 - 1/A$. As we will see, the accepted sample is distributed according to $\D$. This is because, conditional on acceptance, its mass is exactly that assigned by $\D$.
Specifically, we claim that the output of Line \ref{line:out} of \Cref{algo:total-exact} has law $\D$. Observe that the whole stochastic process of \Cref{algo:total-exact} outputs $x \in [n]$ with probability
\begin{align*}
\sum_{i = 0}^{\infty} \Pr[\mathrm{Reject}]^i \Pr[x \text{ is Accepted}] 
&= \sum_{i = 0}^{\infty} \Pr[\mathrm{Reject}]^i \cdot \frac{(\D(x)/ \wt{\D}(x)) \cdot \wt{\D}(x)}{ \sum_{y \in [n]} \D(y)/\wt{\D}(y)}\\ 
& = \frac{\D(x)}{ \sum_{y \in [n]} \D(y)/\wt{\D}(y)} \sum_{i = 0}^{\infty} \Pr[\mathrm{Reject}]^i\,. 
\end{align*}
 We have that
$
\Pr[\mathrm{Reject}] = \frac{1}{A} \sum_{y \in [n]} \frac{\D(y)}{\wt{\D}(y)} (1 - \wt{\D}(y)) = 1 - \frac{1}{A} \in (0,1)\,.$
Hence, we have that the whole stochastic process of \Cref{algo:total-exact} outputs $x \in [n]$ with probability
\begin{align*}
\sum_{i = 0}^{\infty} \Pr[\mathrm{Reject}]^i \Pr[x \text{ is Accepted}] 
& = \frac{\D(x)}{ \sum_{y \in [n]} \D(y)/\wt{\D}(y)} \cdot A = \D(x) \,,
\end{align*}
since $\sum_{i \geq 0} \lambda^i = \frac{1}{1-\lambda}$ for $\lambda \in (-1,1).$ As a result, we get that the probability that $x$ is the output of the above stochastic process is $\D(x).$
\end{proof}
\noindent These claims complete the proof. Specifically, the total sample complexity is derived as follows: Since we have to execute the CFTP iterations $\Theta(n)$ times, we should call each CFTP process with $\delta' = C \delta/n$ for some constant $C$. \Cref{proposition:cftp-convergence} guarantees that the CFTP process will terminate using $O(n\log(n)/\lambda(\vec Q) \cdot \log(n/\delta))$ samples with probability $1-\Theta(\delta/n)$. Hence, applying the the union bound for the $\Theta(n)$ CFTP calls and the single call of the learning algorithm, gives that, with probability $1-\delta$, a number of $O(n^2 \log^2(n) /\lambda(\vec Q) \cdot \log(1/\delta))$ samples suffices in order to get a perfect sample from $\D$. 
\end{proof}


\section{The Proof of \Cref{lemma:tranform-properties12} (Properties of Rescaled Random Walk)}
\label{app:down}
In the next section, we discuss some useful steps as a warm-up for the proof of \Cref{lemma:tranform-properties12}. The proof can be found at the \Cref{section:proof-rescale}.
\subsection{Sketch of the Idea}
By downscaling the transition probabilities (as we will see below), we can decouple the Markov chain from $\D$. Then, the transition matrix of the modified Markov chain only depends on the pair distribution $\Q$, and the convergence of the CFTP algorithm is determined by $\lambda(\vec{Q})$. Hence, we can sample efficiently even from target distributions $\D$ that may be multimodal or has many low probability points. E.g., in case of a multimodal stationary distribution $\D$, the spectral properties of the walk would remind these of a disconnected graph and the sample complexity of \Cref{algo:exact-naive} would be quite high. We use the estimate $\wt{\D}$ (which we obtained from our learning algorithm) of the target distribution $\D$ to transform the Markov chain of the Local Sampling Scheme defined in Equation~\eqref{equation:trans-matrix} to a modified Markov chain with an almost uniform stationary distribution. This transformation can be viewed as a downscaling mechanism:
\begin{definition}[Downscaling]
\label{def:downscaling}
Let $ p \in (0,1)$ and let $X \sim \mathrm{Be}(p)$ be a Bernoulli random variable. Let $\lambda \in (0,1).$ The random variable $Y \sim \mathrm{Be}(\lambda p)$ is called a $\lambda$-downscaler of $X$. 
\end{definition}

Applying \Cref{theorem:learning12}, we can consider that, for any $x \in [n]$, there exists a coefficient $\wt{\D}(x) \approx \D(x)$. The idea is to use $\wt{\D}(x)$ and make the stationary distribution of the modified chain close to uniform. Intuitively, this transformation should speedup the convergence of the CFTP algorithm. 

\begin{step}
Implementation of the downscaling and the rescaled matrix $\wt{
\vec M}$.
\end{step}
\noindent We can implement the modified Markov chain via downscaling as follows: Consider an edge $\{x,y\}$ with transition probability pair $(p_{xy}, p_{yx})$. Without loss of generality, we assume that $\wt{\D}(y) > \wt{\D}(x)$ (which intuitively means that we should expect that $p_{xy} > p_{yx}$). Then, the downscaler leaves $p_{yx}$ unchanged and reduces the mass of $p_{xy}$ to make the two transitions almost balanced. Our exact sampling algorithm will perform this downscaling phase to the matrix $\vec M$ (see Equation \eqref{equation:trans-matrix}).

Consider the transition matrix $\vec M$ with $\vec M_{xy} = \Q(x, y) \frac{\D(y)}{\D(x)+\D(y)}$ and $\vec M_{xx} = 1 - \sum_{y \neq x} \vec M_{xy}$.  Also, let $\wt{\D}$ be an estimate for the distribution $\D$ with some sufficiently small accuracy $\epsilon$, to be chosen (see also \Cref{theorem:learning12}). 
For the pair $(x,y)$, we modify the transition probability $p_{xy} := \vec M_{xy}$, only if $p_{xy} > p_{yx}$, to be equal to the following 
\[ 
\wt{p}_{xy} = \Q(x, y) \frac{\D(y)}{\D(x)+\D(y)} \frac{\wt{\D}(x)}{\wt{\D}(y)} \approx \Q(x, y) \frac{\D(x)}{\D(x)+\D(y)} = p_{yx}\,, \]
where we use that $\wt{\D}(x) \approx \D(x)$ and $\wt{\D}(y) \approx \D(y)$. 
The transition probability from $x$ to $y$ corresponds to a Bernoulli variable $\mathrm{Be}(p_{xy})$, which is downscaled by $\wt{\D}(x)/\wt{\D}(y) < 1$. The modified transition probability $\wt{p}_{xy}$ can be implemented by drawing a $\Lambda \sim \mathrm{Be}(\wt{\D}(x)/\wt{\D}(y))$ and then drawing a $P \sim \mathrm{Be}(p_{xy})$ (from $\mathrm{Samp}(\Q; \D)$) and, finally, realizing the transition from $x$ to $y$ only if $\Lambda P = 1$. This implementation is valid since the two sources of randomness are independent. 

The modified transition matrix $\wt{\vec{M}}$ can be written as:
\begin{equation}
\label{equation:downscaled}
    \wt{\vec{M}} =\wt{  \vec D} \otimes \vec{Q} + \vec{Q} \circ [\epsilon_{xy}]\,,
\end{equation}
where 
$\wt{\vec D}$ is a symmetric matrix with $\wt{\vec D}_{xy} = \mathrm{min} \{\D(x), \D(y)\} / (\D(x) + \D(y)) = \wt{\vec D}_{yx}$ and
$\vec Q$ is the Laplacian matrix of the graph associated with the $\mathrm{Samp}(\Q; \D)$ and $\vec{Q} \circ [\epsilon_{xy}]$ denotes the modified Hadamard product\footnote{We have that $\Big( \vec{Q} \circ \vec A \Big)_{xy} = \Q_{xy} A_{xy}, \Big( \vec{Q} \circ \vec A \Big)_{xx} = -\sum_{y \neq x} \Q_{xy} A_{xy} .$} between the Laplacian $\vec{Q}$ and the matrix with the estimation error $\epsilon_{xy}$ is only non-zero in the modified transitions $x \to y$, i.e., if $p_{xy} > p_{yx}$ then $\wt{\vec M}_{xy} = \Q(x,y) \cdot( \frac{\D(x)}{\D(x) + \D(y)} +  \epsilon_{xy})$. For this pair, we set $\epsilon_{yx} = 0$. We denote this error matrix with $[\epsilon_{xy}]$. 

Specifically,
for a pair $(x,y)$, before the downscaling, the pairwise $(x,y)$-comparison corresponds to the random coin $\left(\frac{\D(y)}{\D(x) + \D(y)}, \frac{\D(x)}{\D(x) + \D(y)}\right)$ and let $\D(y) > \D(x)$. After the downscaling, this coin becomes (locally) almost fair, i.e.,  $(\frac{\D(x)}{\D(x) + \D(y)} + \epsilon_{xy}, \frac{\D(x)}{\D(x) + \D(y)}).$ 

\begin{step}
Obtaining a simpler matrix $\wt{\vec M}$ with absolute spectral gap of the same order.
\end{step}

Since the coins are locally fair, we can work with the following matrix that has also almost uniform stationary distribution and conductance of the same order (since making each transition $(x,y)$ more lazy (i.e., increasing the probability of $x \to x$ and $y \to y$ by some \emph{constant}) and still keeping the transitions $x \to y$ and $y \to x$ (almost) equal cannot significantly affect the conductance, i.e., the desired symmetry is preserved). Let $\wt{\vec M} = \vec I - c \cdot \vec Q + \vec Q \circ [\epsilon_{xy}]$ for some \emph{constant} $0 \leq c \leq 1/2$. This matrix with $c = \min_{(x,y) \in \calE} \D(x)/(\D(x) + \D(y))$ can be obtained by further performing downscaling at each transition (and $c$ is constant due to \Cref{assumption:efficsample}); then each transition will be equal to the minimum global transition probability (potentially with some $O(1/\sqrt{n})$ noise term). For simplicity, we let $c = 1/2$. We get the following matrix:
\begin{equation}
\label{equation:balanced}
\wt{\vec M}_{xy} = \Q(x,y) \Big( \frac{1}{2} + \epsilon_{xy} \Big),~\wt{\vec M}_{yx} = \Q(x,y) \Big( \frac{1}{2} + \epsilon_{yx} \Big)\,, 
\end{equation}
and
\[
\wt{\vec M}_{xx} = 1 - \sum_{y \neq x} \wt{\vec M}_{xy} = 1 - \frac{1}{2} \sum_{y \neq x}\Q(x,y) -  \sum_{y \neq x}\epsilon_{xy} \Q(x,y) \,.
\]
Also, we observe that the generator of the Markov chain $\vec{I} - \wt{\vec{M}}$ is close up to scaling to the Laplacian $\vec{Q}$. This explains why the convergence time of the Markov chain with transition matrix $\wt{\vec{M}}$ spectrally depends only on the distribution $\Q$. Specifically, the Laplacian $\vec{Q}$ corresponds to the dominant component for the convergence of the algorithm and the Hadamard product is the low-order noise induced by the chain transformation. As we will see, the larger the smallest non-zero eigenvalue of the Laplacian matrix and, hence, the larger the spectral gap of the transition matrix of the transformed chain, the faster the Markov chain converges to its stationary distribution. 

\subsection{The Proof of \Cref{lemma:tranform-properties12}}
\label{section:proof-rescale}
The following lemma summarizes the key properties of the downscaled random walk.
\begin{lemma}
[Properties of Rescaled Random Walk]
\label{lemma:tranform-properties}
Let $\D$ be a distribution on $[n]$ and consider an $\epsilon$-relative approximation $\wt{\D}$ of $\D$, as in \Cref{theorem:learning12} with $\epsilon = O(1/\sqrt{n})$. Consider the transition matrix $\vec{M}$ of the Local Sampling Scheme $\mathrm{Samp}(\Q;\D)$ (see Equation~\eqref{equation:trans-matrix}), and let $\wt{\vec{M}}$ be the $\wt{\D}$-scaling of $\vec{M}$ (see Equation~\eqref{equation:downscaled}). Then, the following hold.
\begin{itemize}
    \item[$(i)$] The transition matrix $\wt{\vec{M}}$ has stationary distribution $\wt{\pi}_0(x) = \Theta(1/n)$ for all $x \in [n]$.
    \item[$(ii)$] The absolute spectral gap $\Gamma(\wt{\vec{M}})$ of $\wt{\vec{M}}$ and the minimum non-zero eigenvalue $\lambda(\vec{Q})$ of the Laplacian matrix $\vec{Q}$ satisfy $\Gamma(\wt{\vec{M}}) = \Omega(\lambda(\vec{Q}))$. 
    \item[$(iii)$] For any $\epsilon_0 \in (0,1)$, the mixing time of the transition matrix $\wt{\vec{M}}$ is $T_{\mathrm{mix}}(  \wt{\vec{M}}; \epsilon_0) = O\left(\frac{\log(n/\epsilon_0)}{\lambda(\vec{Q})} \right)$. 
\end{itemize}
\end{lemma}
In the above, we can choose $\epsilon_0 = 1/4$. The proof of the \Cref{lemma:tranform-properties} follows.


\label{appendix:properties}
\begin{proof} [Proof \emph{of}~\Cref{lemma:tranform-properties}]
We remark that it suffices to work with the matrix of Equation~\eqref{equation:balanced}, since the results will only change by at most some constant. Observe that the spectral gap of the matrix of Equation \eqref{equation:downscaled} is at least as high as one of the matrix $\vec I - c \cdot \vec Q + \vec Q \circ [\epsilon_{xy}]$ with constant $c = \min_{(x,y) \in \calE} \D(x)/(\D(x) + \D(y)) = \Theta(\phi)$ and this matrix has spectral gap of the same order as the matrix of Equation \eqref{equation:balanced}.
We break the proof into three claims.
\begin{claim}
The transition matrix $\wt{\vec{M}}$ has stationary distribution $\wt{\pi}_0$, that satisfies $\wt{\pi}_0(x) = \Theta(1/n)$ for any $x \in [n]$.
\end{claim}
\label{claim:unif}
\begin{proof}
\noindent The chain $\wt{\vec M}$ remains irreducible and the detailed balance equations of the matrix $\wt{\vec M}$ satisfy:
\begin{equation*}
    \frac{\wt{p}_{xy}}{\wt{p}_{yx}} = \frac{\Q(x,y) \D(y)/ \wt{\D}(y) }{\Q(y,x) \D(x)/ \wt{\D}(x)} \,.
\end{equation*}
Since $\wt{\D}$ is an $\epsilon$-relative approximation of $\D$ with $\epsilon = O(1/\sqrt{n})$, it holds that for any $x \in [n]$:
\begin{equation*}
    \D(x) / \wt{\D}(x) \in [1-\epsilon, 1+\epsilon]\,,
\end{equation*}
and hence the unique stationary distribution $\wt{\pi}_0$ satisfies: $\wt{\pi}_0(x)/\wt{\pi}_0(y) \in [1-\epsilon, 1+\epsilon]$ for any $x,y \in [n]$. So, this holds for $x = \argmax \wt{\pi}_0$ and $y = \argmin \wt{\pi}_0$ and, since it should hold that $\sum_x \wt{\pi}_0(x) = 1$, it must be the case that $\wt{\pi}(x) = \Theta(1/n)$.
\end{proof}

\begin{claim}
For the absolute spectral gap $\Gamma(\wt{\vec M})$ of $\wt{\vec M}$ and the minimum non-zero eigenvalue $\lambda(\vec{Q})$ of the Laplacian matrix $\vec{Q}$, it holds that $\Gamma(\wt{\vec M}) = \Omega(\lambda(\vec{Q}))$. 
\end{claim}
\begin{proof}
\noindent For $x \neq y$, there exist error estimates $\epsilon_{xy}$ and $\epsilon_{yx}$; one of them is zero and the other's absolute value is of order $O(1/\sqrt{n}).$ 
Our goal is to control the absolute spectral gap $\Gamma(\wt{\vec M})$. Recall that the transition matrix can be written as follows
\begin{equation*}
    \wt{\vec M} = \vec{I} - \frac{1}{2} \vec{Q} +  \vec{Q} \circ [\epsilon_{xy}]\,,
\end{equation*}
where $\epsilon_{xy} \in [-\epsilon, \epsilon]$ for some $\epsilon = O(1/\sqrt{n})$. Note that $[\epsilon_{xy}]$ denotes the $n \times n$ error matrix and the operator $\circ$ denotes the standard Hadamard product. Note that the matrix $\vec{I} - \frac{1}{2} \vec{Q}$ is symmetric, while $\vec{Q} \circ [\epsilon_{xy}]$ needs not to be.

Let $\lambda_1 > \lambda_2 \geq \dots \geq \lambda_n$ be the spectrum of the transition matrix $\wt{\vec M}$. Also, the matrix is right stochastic with $\wt{\vec M} \vec{1} = \vec{1}$ and $\lambda_1 = 1$. Let $\Gamma(\wt{\vec M})$ be the absolute spectral gap, that is strictly positive by aperiodicity and irreducibility. The Weyl's Inequality for general matrices (\cite{horn2012matrix, tao2012topics}) describes a multiplicative majorization between the ordered absolute eigenvalues and singular values of $\vec A$ and gives a useful perturbation bound. Recall that $\sigma_i^2(\vec A) = \lambda_i(\vec A \vec A^*) = \lambda_i(\vec A^* \vec A)$ for an arbitrary matrix $\vec A \in \mathbb{C}^{m \times n}$.
\begin{fact}
[Weyl's Inequality (see~\cite{horn2012matrix})]
\label{fact:weyl}
Consider the matrices $\vec A, \vec E \in \mathbb{C}^{n \times n}$. Define the singular value of $\vec A$ in decreasing order (counting multiplicity) $\sigma_1(\vec A) \geq \ldots \geq \sigma_n(\vec A) \geq 0$. Then, for $k=1,\dots,n$, the following hold
\begin{itemize}
    \item $\prod_{i \in [k]}|\lambda_i(\vec A)| \leq \prod_{i \in [k]} \sigma_i(\vec A)$ with $|\lambda_1(\vec A)| \geq |\lambda_2(\vec A)| \geq \dots |\lambda_n(\vec A)|$.
    \item $|\sigma_k(\vec A + \vec E) - \sigma_k(\vec A) | \leq \|\vec  E \|_{2}$\,.
\end{itemize}
\end{fact}
For the matrix $\wt{\vec M}$, we have that $|\lambda_1| = 1$ and the second largest absolute eigenvalue is $\max\{|\lambda_2|, |\lambda_n|\}$. We can apply the first property for the singular values from the Fact \ref{fact:weyl} and get that
\begin{equation*}
    \Gamma(\wt{\vec M}) = \lambda_1 - \max\{|\lambda_2|, |\lambda_n|\} \geq 1 - \sigma_1(\wt{\vec M})\sigma_2(\wt{\vec M}) \,,
\end{equation*}
where $\sigma_1, \sigma_2$ correspond to the two largest singular values. Observe that the largest singular value is $\sigma_1(\wt{\vec M}) = \max_{\|\vec{v}\|_2 = 1}\|\wt{\vec M} \vec{v}\| =  1$ and, using the second property of Fact \ref{fact:weyl}, we can control the perturbation of the second largest singular value
\begin{equation*}
    \sigma_2 \left (\vec{I} - \frac{1}{2}\vec{Q} \right) - \| \vec{Q} \circ [\epsilon_{xy}] \|_{\mathrm{2}} \leq \sigma_2(\wt{\vec M}) \leq \sigma_2 \left(\vec{I} - \frac{1}{2}\vec{Q} \right) + \| \vec{Q} \circ [\epsilon_{xy}] \|_{\mathrm{2}}\,.
\end{equation*}
In order to lower bound the absolute spectral gap, it suffices to upper bound the second largest singular value.
But, for the second largest singular eigenvalue of the real symmetric matrix $\vec I - \vec Q/2$, it holds that
\begin{equation*}
    \sigma_2 \left (\vec{I} - \frac{1}{2}\vec{Q} \right) = \max \left\{\left|\lambda_2 \left(\vec{I} - \frac{1}{2}\vec{Q} \right)\right|, \left|\lambda_n \left(\vec{I} - \frac{1}{2}\vec{Q}\right)\right|\right\} \,.
\end{equation*}
Since the matrix $\vec I - \vec Q/2$ is symmetric, we have that
\begin{equation*}
    \sigma_2 \left(\vec{I} - \frac{1}{2}\vec{Q} \right) = \max \left\{\Big|\max_{\vec{v} \perp \vec{1}, \| \vec{v} \|_2 = 1} \vec{v}^\top \left(\vec{I} - \frac{1}{2}\vec{Q}\right) \vec{v}\Big|, \Big|\min_{\| \vec{v} \|_2 = 1} \vec{v}^\top \left(\vec{I} - \frac{1}{2}\vec{Q}\right) \vec{v}\Big|\right\} \,,
\end{equation*}
and, hence, if we let $0 = \mu_1 \leq \mu_2 \leq \ldots \leq \mu_n$ be the spectrum of the Laplacian matrix $\vec Q$, we have that
\begin{equation*}
\sigma_2 \left(\vec{I} - \frac{1}{2}\vec{Q} \right) =   
\max \Big\{  \Big| 1 - \frac{1}{2} \mu_2(\vec{Q}) \Big|, \Big| 1 - \frac{1}{2} \mu_n(\vec{Q}) \Big| \Big\} \,.
\end{equation*}
For the Laplacian matrix $\vec Q$, observe that $|\vec Q_{xx}| = |\sum_{y \sim x}\Q(x,y)| \leq 1$. Hence, from Gershgorin's Theorem (see \Cref{lemma:gersgorin}), we get that all the eigenvalues of the Laplacian lie on the real axis (since the matrix is symmetric) and all the disks $B(x,r)$ are centered at points $|x| \leq 1$ and the radii $r$ are upper bounded by $1$. Hence, it holds that $0 \leq \mu_2(\vec Q) \leq \mu_n(\vec Q) \leq 2$.  So, we can take that $\sigma_2(\vec{I} - \frac{1}{2}\vec{Q}) = 1 - \frac{1}{2} \mu_2(\vec{Q}) \stackrel{\mathrm{def}}{=} 1 - \frac{1}{2} \lambda(\vec{Q})  \geq 0$, where $\lambda(\vec Q)$ denotes the smallest non-zero eigenvalue of the Laplacian matrix $\vec Q$. 

Now, we have that $     \| \vec{Q} \circ [\epsilon_{xy}] \|_{\mathrm{2}} = \max_{\| \vec{v}\|_2 = 1} \| (\vec{Q} \circ [\epsilon_{xy}]) \vec{v} \|_2  = O(1/\sqrt{n}) \cdot \mu_n(\vec Q) = O(1/\sqrt{n})$ and so we get 
$\sigma_2(\wt{\vec M}) \leq \left( 1 - \Theta(\lambda(\vec{Q}))\right) + O(1 / \sqrt{n})$. This implies that
\[ 
\sigma_2(\wt{\vec M}) \lesssim \max \{ 1 - \Theta(\lambda(\vec{Q})), O(1 / \sqrt{n}) \} \,,
\]
and so
\[
\Gamma(\wt{\vec M}) \gtrsim 1 -  \max \{ 1 - \Theta(\lambda(\vec{Q})), O(1 / \sqrt{n}) \}
\]
This implies that $\Gamma(\wt{\vec M}) = \Omega(\lambda(\vec{Q}))$ for sufficiently large $n$.
The above proof will be similar for any matrix $\wt{\vec M} = I - c \cdot \vec Q + \vec Q \cdot [\epsilon_{xy}]$ for $0 \leq c \leq 1/2$ which is obtained by the discussion of \textbf{Step 2}. Such a matrix has an absolute spectral gap that is, on the one side, lower bounded by $\lambda(\vec Q)$ (as we showed above) and, on the other side, upper bounded by the absolute spectral gap of the original matrix of Equation \eqref{equation:downscaled}. This completes the proof.
\end{proof}

We continue with the mixing time result of the transformed Markov chain. Since we have shown that the absolute spectral gap is $\Omega(\lambda(\vec Q))$, we can directly get the desired result, whose proof relies on the analysis of \Cref{lemma:mixing-time-peres} and can be found, e.g., in~\cite{levin2017markov}. 

\begin{claim}
The $\epsilon_0$-mixing time of the transition matrix $\wt{\vec M}$ is equal to
    \[
      T_{\mathrm{mix}}(\wt{\vec M}; \epsilon_0) = O(\log(n/\epsilon_0)/\lambda(\vec{Q}))\,.
    \]
\end{claim}
In the above claim, we choose $\epsilon_0 = 1/4$.

\noindent Combining these claims, the proof is completed. 
\end{proof}

\section{ The Proof of \Cref{theorem:learning12} (Main Result of Learning Phase)}
\label{sec:learning-phase}
The first step of \Cref{algo:total-exact} is to learn the target distribution $\D$ in $\epsilon$-relative error for some $\epsilon > 0$ and pass it as input to  \Cref{algo:exact}. 
In this section, we will provide the learning results for abstract $\epsilon$; however, our algorithm applies these results with $\epsilon = 1/\sqrt{n}$.
For two distributions $\D, \wt{\D}$ with ground set $[n]$, we introduce the sequence/list (of length $n$) $1 - \D/\wt{\D} := (1 - \D(x)/\wt{\D}(x))_{x \in [n]}.$ Observe that the pair of sequences $(1-\D/\wt{\D}, 1-\wt{\D}/\D)$ captures the relative error between the two distributions.
The sample complexity of the task of learning $\D$ in $\epsilon$-relative error is summarized by the following (restatement of \Cref{theorem:learning12}): 

\begin{theorem*}
For any $\epsilon, \delta > 0$, there exists an algorithm (\Cref{algo:learn}) that draws $N  = O\left(\frac{n}{\lambda(\vec{Q})\epsilon^2}\log(\frac{1}{\delta}) 
\right)$ samples from a Local Sampling Scheme $\mathrm{Samp}(\Q;\D)$ satisfying Assumptions~\ref{assumption:identifiability} 
and~\ref{assumption:efficsample}, runs in time polynomial in $N$, and, with probability at least $1-\delta$, computes an estimate $\wt{\D}$ of the target distribution $\D$, that satisfies the following relative error guarantee $ \max \big \{ \| 1- \D/\wt{\D} \|_{\infty}, \|1-\wt{\D}/\D\|_{\infty} \big \} \leq \epsilon$.
\end{theorem*}

\begin{algorithm}[ht] 
\caption{Learn using Shifting from (Pairwise) Local Sampling Schemes}
\label{algo:learn}
\begin{algorithmic}[1]
\Procedure{Learn-Shift}{$\epsilon, \delta$} \Comment{\emph{Sample access to oracle $\mathrm{Samp}(\Q; \D)$}.}
    \State Set $N = \Theta\left(\frac{n}{\lambda(\vec{Q})\epsilon^2}\log(\frac{1}{\delta}) \right)$. 
    \State Draw $N$ samples of the form $((x_i, y_i),q_i) \in \calE \times \{0,1\}.$ 
    \State Obtain $\wh{\vec z}$ using \Cref{algo:learn-shah} (the estimation vector satisfies $\< \vec 1, \wh{\vec z} \> = 0$ and must be shifted). 
    \State Compute $C = \log \left(\sum_{x \in [n]} \exp(\wh{z}_x) \right)$ 
    \State Set $z_x' = \wh{z}_x - C$ for any $x \in [n]$ \Comment{\emph{See \Cref{sec:distr-learn}}.}
    \State Output $\vec z'$ \Comment{\emph{The output satisfies $\| \vec z - \vec z' \|_{\infty} \leq \epsilon$}.}
\EndProcedure
\end{algorithmic}
\end{algorithm}

Let us sketch the proof of \Cref{theorem:learning12}. For an arbitrary weight vector $\vec w \in \reals^n_{>0}$, we use the re-parameterization $z_i = \log(w_i)$. When in addition $\vec w \in \Delta^n$, i.e., $\vec w$ is a probability distribution over $[n]$ and is usually denoted by $\D$, we call $\vec z$ the natural parameter vector of $\D$. Recall that $\vec Q$ is a Laplacian matrix where $\vec Q_{xy} = -\Q(x,y)$, 
i.e., it is the Laplacian matrix of the 
graph $G_{\Q}$ 
weighted by the mass assigned by $\Q$.
The proof goes as follows: we draw i.i.d. samples from the LSS. We first apply the following result which is a modification of the results of \cite{shah2016estimation}.
\begin{theorem}[Variant of \cite{shah2016estimation}]
\label{thm:shah-variant}
Let $G_{\Q}$ be the graph of the support $\calE$ of $\Q$ satisfying \Cref{assumption:identifiability} with $|V| = n$ vertices and associated weighted Laplacian matrix $\vec Q$ (see Equation~\eqref{eq:laplacian-q}). Let $\vec L$ be the associated empirical Laplacian matrix with $\E[\vec L] = \vec Q$. Consider a vector  $\vec{z} \in \reals^n$ with $\< \vec 1, \vec z\> = 0$ satisfying the constraint $\max_{x,y \in \calE}| z_x - z_y| \leq \log(\phi)$ for some constant $\phi > 1$.
Let $\mathrm{Ex}(\Q; \vec z)$ be the oracle that generates the example $\{(x,y),q\}$ as follows: the edge $(x,y)$ is chosen with probability $\Q(x,y)$ and the bit $q$ is set to $1$ with probability $\exp(z_x)/(\exp(z_x) + \exp(z_y))$; otherwise it is set to $0$.
For any $\epsilon, \delta > 0$, there is a maximum likelihood estimator of $\vec{z}$ (\Cref{algo:learn-shah}) which draws $N = O\left(\frac{n}{\lambda(\vec{L})\epsilon^2}\log(\frac{1}{\delta}) 
\right)$ samples from $\mathrm{Ex}(\Q;\vec z)$
and
computes,
in time polynomial in the number of samples $N$,
an estimate $\wh{\vec{z}}$ such that $\| \vec{z} - \wh{\vec{z}} \|_{2} \leq \epsilon$, with probability at least $1-\delta$. 
\end{theorem}
The proof can be found at \Cref{proof:thm:shah}. Some comments are in order:
\begin{enumerate}
    \item In the work of \cite{shah2016estimation}, the target is the (re-parameterized) weights vector $\vec z^\star \in \reals^n$ satisfying the conditions $\< \vec 1, \vec z^\star \> = 0$ and $\| \vec z^\star \|_{\infty} \leq B$ for some constant $B$. The provided algorithm minimizes the empirical log-likelihood and, using $O(n/(\lambda(\vec Q) \cdot \epsilon^2)))$ samples, computes an estimate $\wh{\vec z}$ so that $\| \vec z^\star - \wh{\vec z}\|_2 \leq \epsilon$.
    \item Our provided \Cref{algo:learn-shah} has the same guarantees as the algorithm of \cite{shah2016estimation} but the target weight vector satisfies the conditions $\< \vec 1, \vec z^\star \> = 0$ and $ | z^\star_x - z^\star_y | \leq B$ for some constant $B$. The algorithm draws sufficiently many samples and then minimizes the empirical negative log-likelihood objective over an appropriately selected constrained set (see \Cref{appendix:learningalgo}). Observe that for a Local Sampling Scheme with distribution $\D$, the natural parameter vector $\vec z^\star$ does not satisfy $\< \vec 1, \vec z^\star\> = 0$ and so \Cref{thm:shah-variant} and \Cref{algo:learn-shah} cannot be directly applied.
    \item We will discuss (\Cref{sec:distr-learn}) how to apply \Cref{algo:learn-shah} to distributions, i.e., target weight vectors with $\< \vec 1, \vec z^\star \> \leq 0.$ This results in learning the target vector $\vec z^\star$ in $L_{\infty}$ norm (which is sufficient for \Cref{algo:total-exact}).
\end{enumerate}


The algorithm of \Cref{thm:shah-variant} follows.
\begin{algorithm}[ht] 
\caption{Learning from (Pairwise) Local Sampling Schemes}
\label{algo:learn-shah}
\begin{algorithmic}[1]
\Procedure{Learn}{$\epsilon, \delta$} \Comment{\emph{Sample access to oracle $\mathrm{Ex}(\Q; \vec z)$}.}
    \State Set $N = \Theta\left(\frac{n}{\lambda(\vec{Q})\epsilon^2}\log(\frac{1}{\delta}) \right)$. 
    \State Draw $N$ samples of the form $((x_i, y_i),q_i) \in \calE \times \{0,1\}.$ 
    \State Compute the empirical negative log-likelihood objective
    \[
L_N(\vec z; \{ ((x_i, y_i),q_i) \}_{i \in [N]}) = -\frac{1}{N}\sum_{i = 1}^{N} 
\vec 1\{ q_i = 1\} z_{x_i} + 
\vec 1\{ q_i = 0\} z_{y_i}
-\log(\exp(z_{x_i}) + \exp(z_{y_i}))\,.
\]
    \State Minimize $L_N$ using gradient descent in the subspace $\Omega_{\phi}$. \Comment{\emph{See \Cref{appendix:learningalgo}}.}
    \State Output the guess vector $\wh{\vec z}$.
\EndProcedure
\end{algorithmic}
\end{algorithm}
To complete the proof of \Cref{theorem:learning12}, we combine the upcoming \Cref{lemma:concetration-fiedler} and \Cref{lemma:stability} with the main result of~\cite{shah2016estimation}. 
\Cref{lemma:concetration-fiedler} states that, after drawing sufficiently many samples from $\mathrm{Samp}(\Q; \D)$, we can construct an empirical Laplacian matrix $\vec L$, whose Fiedler eigenvalue is of the same order as the one of the unknown matrix $\vec{Q}$ of the Local Sampling Scheme. Moreover, notice that  \Cref{lemma:concetration-fiedler} implies that, with high probability, the graph induced by the matrix $\vec L$ is connected, i.e., $\lambda(\vec L) > 0$. 

\begin{lemma}
[Concentration of Empirical Fiedler Eigenvalue]
\label{lemma:concetration-fiedler}
Let $\epsilon > 0$ and $V = \{ \vec{v} \in \{\pm 1, 0\}^n : \vec{v} = \vec{e}_i - \vec{e}_j \text{ for any } i < j \}$ and let $\Q$ be a distribution over $V$. Let $\vec{Q}$ be the Laplacian matrix of $\Q$ (see Equation~\eqref{eq:laplacian-q}). There exists an algorithm that uses $O(\log(n/\delta)/(\lambda(\vec{Q}) \epsilon^2))$ samples from $\Q$ and computes a matrix $\vec L$ that satisfies
$|\lambda(\vec L) - \lambda(\vec{Q})| \leq \epsilon \lambda(\vec{Q})$,
with probability at least $1-\delta$, where $\lambda(\cdot)$ denotes the second smallest eigenvalue of a Laplacian matrix.
\end{lemma}

\begin{proof}
Let $\{\vec X_i = \vec{v}_i \vec{v}_i^T\}$ be a finite sequence of $m$ independent symmetric square matrices (of common dimension $n$) with $\vec{v}_i \sim \Q$, let $\vec L = \frac{1}{m}\sum_{i=1}^m \vec X_i$, and let $M = \lambda_{\max}(\vec X_i) =  \Theta(1)$. To deal with the second smallest eigenvalue of the Laplacian matrix, we can project the Laplacian matrix $\vec L$ to the orthogonal complement of the vector $\vec{1}$. Under this transformation, we can obtain a matrix $\vec L'$ that has two crucial properties. First, it still is a sum of independent positive semidefinite terms and its minimum eigenvalue coincides with the Fiedler eigenvalue of $\vec L$. To achieve this, as in~\cite{tropp2015introduction}, we introduce the transformation $\vec R \in \mathbb{R}^{(n-1) \times n},$ that satisfies $\vec R \vec R^T = \vec I_{n-1}$ and $\vec R \vec{1} = \vec{0}$. Hence, it holds that $\vec L' = \vec R \vec L \vec R^T$ and $\E [\vec L'] = \vec R \vec{Q} \vec R^T$. Then, 
\[
\Pr_{\vec v_{1..m} \sim \Q^{\otimes m} }\Big [\lambda(\vec L) \leq (1-\epsilon) \lambda(\vec Q)\Big ] = \Pr_{\vec v_{1..m} \sim \Q^{\otimes m} }\Big [\lambda_{\min}(\vec L') \leq (1-\epsilon) \lambda_{\min}(\vec R \vec Q \vec R^T)\Big ] \,.
\]
Using \Cref{prop:conc}, we get that:
\[
\Pr_{\vec v_{1..m} \sim \Q^{\otimes m} }\Big [\lambda(\vec L) \leq (1-\epsilon) \lambda(\vec Q)\Big ] \leq n \left(e^{-\epsilon}/(1-\epsilon)^{1-\epsilon} \right)^{\lambda_{\min}(\vec R \vec Q \vec R^T)/ M} \leq n \exp\left (-\epsilon^2 m \frac{\lambda(\vec{Q})}{2M}\right )\,.
\]
To upper bound this probability by $\delta$, it suffices to draw  $O(\log(n/\delta)/(\epsilon^2 \lambda(\vec{Q})))$  samples. So, we get that, with probability at least $1-\delta$, the Fiedler eigenvalue of the empirical matrix $\vec L$ satisfies $\lambda(\vec L) \in (1 \pm \epsilon)\lambda(\vec Q)$. Moreover, the second smallest eigenvalue of the empirical Laplacian matrix $\vec L$ is strictly positive and thus, the induced graph is connected (recall that the number of connected components in the graph is the dimension of the nullspace of the Laplacian matrix). 
\end{proof}

The next lemma guarantees that learning the natural parameters of any distribution in $L_{\infty}$ norm is sufficient for learning the distribution with small relative error over its support. 


\begin{lemma}[Stability of Relative Error]\label{lemma:stability}
Let $\D_1, \D_2$ be discrete distributions supported on the ground set $[n]$ and let $\vec{z}_1, \vec{z}_2$ be the corresponding natural parameter vectors. For any sufficiently small accuracy parameter $\epsilon > 0$, if it holds that $\| \vec{z}_1 - \vec{z}_2 \|_{\infty} \leq \epsilon~$, then the distributions $\D_1, \D_2$ are close in relative error, i.e.,
\(
\max \Big \{ \|1 -\D_1/\D_2\|_{\infty}, \| 1- \D_2/ \D_1 \|_{\infty} \Big\} \leq \epsilon
\).
\end{lemma}
\begin{proof}[Proof \emph{of}~\Cref{lemma:stability}]
\label{proof:stability-ratio}
Let $\epsilon > 0$ sufficiently small and assume that  $\| \vec{z}_1 - \vec{z}_2 \|_{\infty} \leq \epsilon$, i.e., 
\[
\max_{x \in [n]} \Big |\log \Big(\frac{\D_1(x)}{\D_2(x)} \Big) \Big| \leq \epsilon\,.
\]
For $x > 0$, it holds that $1 - \frac{1}{x} \leq \log_e(x)$\footnote{This inequality holds when the base of the logarithm is $e$, the base that we have assumed that we work with.}. 
Hence, for any $x \in [n]$, we get that
\[
1 - \frac{\D_2(x)}{\D_1(x)} \leq \log \Big(\frac{\D_1(x)}{\D_2(x)} \Big) \leq \epsilon\,,
\]
and
\[
1 - \frac{\D_1(x)}{\D_2(x)} \leq \log \Big(\frac{\D_2(x)}{\D_1(x)} \Big) \leq \epsilon\,.
\]
This gives that
\[
\frac{\D_1(x)}{\D_2(x)} \geq 1-\epsilon ~\text{ and }~ \frac{\D_2(x)}{\D_1(x)} \geq 1-\epsilon\,.
\]
This implies that both ratios are upper bounded by $1+\epsilon$. For contradiction, assume that
\[
\frac{\D_1(x)}{\D_2(x)} > 1+\epsilon \iff \frac{\D_2(x)}{\D_1(x)} < \frac{1}{1+\epsilon}\,.
\]
But, the Taylor expansion of the function $x \mapsto \frac{1}{1+x}$ is equal to $1-x+O(x^2)$ for $|x| < 1$. Hence, for sufficiently small $\epsilon,$ the desired bound follows.
\end{proof}

To wrap up, the proof of \Cref{theorem:learning12} goes as follows:
\begin{proof}
[Proof \emph{of} \Cref{theorem:learning12}]
Assume that $\vec z^\star \in \reals^n$ is the true natural parameter vector for the target distribution $\D$. By \Cref{lemma:stability}, in order to obtain the desired relative error, it suffices to control $\vec z^\star$ in $L_{\infty}$. Also, let $\vec z \in \reals^n$ be a shifted variant of $\vec z^\star$ so that $\< \vec 1, \vec z\> = 0$, i.e., $\vec z^\star = \vec z + C \vec 1$ for some $C$.
First, the algorithm draws $N = \Theta\left(\frac{n}{\lambda(\vec{Q})\epsilon^2}\log(\frac{1}{\delta}) \right)$ i.i.d. samples from the LSS with Laplacian matrix $\vec Q$. Using the concentration of the Fiedler eigenvalue (\Cref{lemma:concetration-fiedler}), we have that this number of samples is sufficient to apply \Cref{thm:shah-variant}. This will yield a vector $\wh{\vec z}$ so that $\| \vec z - \wh{\vec z} \|_2 \leq \epsilon$ with probability $1-\delta$. We then apply the transformation described in \Cref{sec:distr-learn} (which corresponds to Line 5-6 of \Cref{algo:learn}). This yields the desired $L_{\infty}$ bound and completes the proof.
\end{proof}




\section{The Proof of \Cref{thm:shah-variant} (Variant of \cite{shah2016estimation})}

In this section, we prove a slight variant of the algorithm of \cite{shah2016estimation} where the target weight vector satisfies the conditions $\< \vec 1, \vec z^\star \> = 0$ and $ | z^\star_x - z^\star_y | \leq B$ for some constant $B$.

\label{app:learning-phase}

\subsection{Description of the Learning Algorithm for Abstract Weight Vectors}
\label{appendix:learningalgo}
\label{proof:thm:shah}

In this section, we provide a complete description of the learning algorithm of~\cite{shah2016estimation} (with a slight modification). We remark that the technical steps required in order to establish our result follow the analysis of \cite{shah2016estimation}.
It suffices to present an efficient learning algorithm that estimates the target weights vector $\vec z^{\star} = (z_1^{\star},\ldots, z_n^{\star})$ in $L_{\infty}$ norm. \emph{We underline that for this section, the target vector $\vec z^\star$ is an abstract weight vector and not the natural parameter vector induced by a distribution.} We assume that $\vec z^\star$ lies in the subspace $\Omega_{\phi}, $ defined in Equation \eqref{eq:space}.
Hence, it suffices that the learning algorithm computes an estimate $\wh{\vec z}$ such that $\| \vec z^{\star} - \wh{\vec z} \|_2 \leq \epsilon$, with high probability, for some accuracy parameter $\epsilon > 0$.

\begin{proof}
[Proof \emph{of} Theorem~\ref{thm:shah-variant}] The algorithm minimizes the empirical negative log-likelihood over a subspace $\Omega_{\phi} \subseteq \mathbb{R}^n$ of the natural parameter vectors. The empirical negative log-likelihood objective with $N$ draws  from the Local Sampling Scheme corresponds to the function 
\[
L_N(\vec z; \{ ((x_i, y_i),q_i) \}_{i \in [N]}) = -\frac{1}{N}\sum_{i = 1}^{N} 
\vec 1\{ q_i = 1\} z_{x_i} + 
\vec 1\{ q_i = 0\} z_{y_i}
-\log(\exp(z_{x_i}) + \exp(z_{y_i}))\,,
\]
where $x_i, y_i \in [n]$ and $q_i \in \{0,1\}$ for any $i \in [N]$.
We optimize this objective using gradient descent in the subspace $\Omega_{\phi}$, where
\begin{equation}
\label{eq:space}
    \Omega_{\phi} = \{ \vec z \in \mathbb{R}^n : \<\vec 1, \vec z\> = 0, \max_{(x,y) \in \calE}|z_x - z_y| \leq \log(\phi) \}\,.
\end{equation}
The first constraint $\<\vec 1, \vec z\> = 0$ is imposed in order to work in the subspace where Laplacian matrices are positive definite (since any Laplacian matrix has its first eigenvalue equal to $0$ with corresponding eigenvector $\vec 1$). The second constraint\footnote{As we mentioned, our learning algorithm almost exactly follows the algorithm of~\cite{shah2016estimation}. The only difference appears in the \emph{second} constraint of the set $\Omega_{\phi}$, where~\cite{shah2016estimation} optimize over vectors with upper bounded $L_{\infty}$ norm. } is equivalent to \Cref{assumption:efficsample}. Since any valid target distribution satisfies $\frac{1}{\phi} \leq \frac{\D(x)}{\D(y)} \leq \phi$ for $(x,y) \in \calE$, we have that
\[
-\log(\phi) \leq z_x - z_y \leq \log(\phi)~\text{ for any } x,y \in \calE\,.
\]
Let $\vec L$ be the empirical Laplacian matrix whose Fiedler eigenvalue $\lambda_2(\vec L)$ is close to the true $\lambda_2(\vec Q)$ (see \Cref{lemma:concetration-fiedler}). We introduce the following notation for the quadratic form $\|\vec v\|_L^2 := \vec v^T \vec L \vec v$ for any vector $\vec v \in \mathbb{R}^n$. The key idea of the algorithm of~\cite{shah2016estimation} is to compute an estimate $\wh{\vec z}$ in the $\| \cdot \|_L$ semi-norm. Afterwards, by the min-max principle for Hermitian matrices, we have that 
\begin{equation}
\label{equation:learningbound}
\| \vec z^{\star} - \wh{\vec z} \|_L^2 \geq \lambda_2(\vec L) \| \vec z^{\star} - \wh{\vec z} \|_2^2\,.
\end{equation}
Hence, the estimation of the true natural parameter vector in $L_2$ norm is directly implied. The analysis for the estimation in the $\vec L$ semi-norm is based on the following fact about $M$-estimators. A proof of this result can be found in~\cite{shah2016estimation}.
\begin{fact}
[see \cite{shah2016estimation}]
\label{fact:Lbound}
Let $\Omega_1 = \{ \vec z \in \mathbb{R}^n : \<\vec 1, \vec z\> = 0\}$ and let $\Omega \subseteq \Omega_1$. Consider the $M$-estimator
\[
\wh{\vec z} \in \argmin_{\vec z \in \Omega} \l(\vec z)\,,
\]
and let $\l$ be a differentiable objective that is $\kappa$-strongly convex\footnote{A function $f : \mathcal{X} \rightarrow \mathbb{R}$ is $\kappa$-strongly convex with respect to a norm $\|\cdot \|$ if, for all $x,y$ in the relative interior of the domain of $f$ and $\lambda \in (0,1)$, we have
\[
f(\lambda x + (1-\lambda)y) \leq \lambda f(x) + (1-\lambda)f(y) - \frac{1}{2}\kappa \lambda(1-\lambda) \| x-y\|^2\,.
\]
If the function $f$ is differentiable, then a second definition for $\kappa$-strong convexity with respect to a norm $\|\cdot \|$ is that for all points $x,y$, we have that $f(y) - f(x) - \<\nabla f(x), y-x\> \geq \kappa \|x-y\|^2$. Recall that the relative interior of a set $S$ is defined as its interior within the affine hull of $S$, i.e., $\mathrm{relint}(S) = \{x \in S : \exists \epsilon > 0, \mathcal{N}_{\epsilon}(x) \cap \mathrm{aff}(S) \subseteq S \}$.
}
at $\vec z^{\star} \in \Omega$ with respect to the $\vec L$ semi-norm. Then, it holds that
\[
\| \vec z^{\star} - \wh{\vec z} \|_L^2 \leq \frac{1}{\kappa} \| \nabla \l(\vec z^{\star}) \|_{L^{\dagger}}\,,
\]
where $\vec L^{\dagger}$ is the Moore-Penrose pseudo-inverse of $\vec L$.
\end{fact}
Using the above result, it suffices to verify that the empirical negative log-likelihood objective is strongly convex in the true parameter vector $\vec z^{\star} \in \Omega_{\phi}$ and to upper bound the dual norm $\| \nabla L_N(\vec z^{\star}) \|_{L^{\dagger}} = \nabla L_N(\vec z^{\star})^T \vec L^{\dagger} \nabla L_N(\vec z^{\star})$. Recall that $\vec L$ is the empirical estimate of the true Laplacian matrix $\vec Q$. We continue with two claims, that are sufficient in order to control the $\vec L$ semi-norm of the vector $\vec z^{\star} - \wh{\vec{z}}.$

\begin{claim}[Strong Convexity at $\vec z^{\star}$]
The empirical negative log-likelihood $L_N(\vec z)$ is $\kappa$-strongly convex at $\vec z^{\star} \in \Omega_{\phi}$ with respect to the $\vec L$ semi-norm with $\kappa = \poly(1/\phi)$. 
\end{claim}

\begin{proof}
It suffices to lower bound the quadratic form $\vec w^T \nabla L_N(\vec z) \vec w$ for all vectors $\vec w \in \reals^n$ and $\vec z \in \Omega_{\phi}$. As in~\Cref{lemma:concetration-fiedler}, we introduce the measurement vector notation $\vec v_i \in \{-1,0,1\}^n$ in order to represent each drawn sample $((x_i, y_i), q_i)$ (recall that the edges in this training set lie in $\calE$ and $q_i \in \{0,1\}$). We set $\vec v_i(x_i) = 2q_i-1, \vec v_i(y_i) = -\vec v_i(x_i)$ and the other coordinates are set to $0$. These measurement vectors are the building blocks of the empirical matrix (see~\Cref{lemma:concetration-fiedler}), since we have that
\[
\vec L = \frac{1}{N} \sum_{i=1}^N \vec v_i \vec v_i^T := \frac{1}{N} \vec X^T \vec X\,. 
\]
Note that the matrix $\vec X \in \{-1,0,1\}^{N \times n}$ has the $i$-th measurement vector $\vec v_i^T$ as its $i$-th row. Using this notation and following the computations of~\cite{shah2016estimation} for the Hessian of $L_N$, we get that
\[
\nabla^2 L_N(\vec z) = \frac{1}{N}\sum_{i=1}^N \Big \{ \vec 1\{v_i = 1\} A_1 + \vec 1\{v_i = 0\}A_0 \Big \} \vec v_i \vec v_i^T\,,
\]
where 
\[
A_1 = (\log F)''(\<\vec z, \vec v_i\>), A_0 = (\log (1-F))''(\< \vec z, \vec v_i\>)~\text{ and }~F(x) = 1/(1 + \exp(-x))\,. 
\]
The function $F(x) = 1/(1 + \exp(-x))$ is the Bradley-Terry function. The function $F$ is strongly-log concave in the interval $[-\log(\phi), \log(\phi)].$ Specifically, we have that
\[
\frac{d^2}{dx^2} \Big( -\log F(x) \Big) = \frac{e^x}{(1+e^x)^2} \geq \frac{\phi}{(1+\phi)^2} =: \kappa(\phi)\,,
\]
since the mapping $x \mapsto \exp(x)/(1+\exp(x))^2$ is symmetric and, so, focusing on the interval $[0,\log(\phi)]$, its minimum is attained at $\log(\phi)$. 

\begin{figure}
    \centering
    \begin{tikzpicture}[scale=1]
\begin{axis}[grid=both, xmin=-5,
          xmax=10,ymax=1/4,
          axis lines=middle,
          enlargelimits]
\addplot[blue,samples=200]  {exp(x)/(1+exp(x))^2} node[above right] {$y= \frac{\exp(x)}{(1+\exp(x))^2}$};
\end{axis}
\end{tikzpicture}
    \caption{The second derivative of the negative logarithm of the  Bradley-Terry function $F$.}
    \label{fig:my_label}
\end{figure}
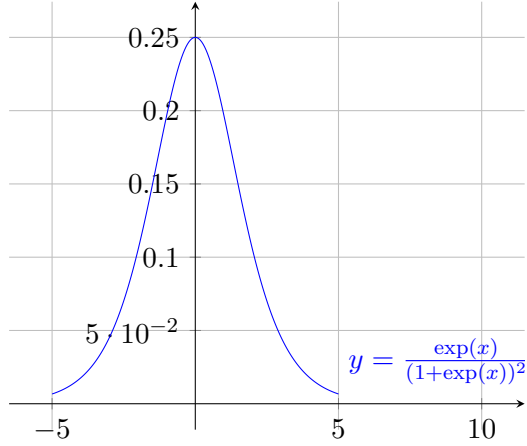


The physical interpretation of these properties is that 
for the pair $i,j$, it does not matter whether $i \succ j$ or $j \succ i$ (symmetry) but if the comparison gap (i.e., $|z_i - z_j|$) attains very large values, then the strong convexity will be very low. Intuitively, note also that, since $\exp(x)/(1+\exp(x))^2 \leq 1/4$, the desired objective is also smooth and it is known that strong convexity and smoothness of a function over a domain of the form $\{ \vec z \in \reals^n : \< \vec 1, \vec z\> = 0\} \cap \mathcal{Z}$, where $\mathcal{Z}$ is convex, imply that the function satisfies the PL inequality over $\mathcal{Z}$. Hence, for this function, gradient-based methods will converge with fast rates (see e.g., \cite{vojnovic2020convergence}).  

\noindent Since any $\vec z \in \Omega_{\phi}$, we conclude that for any $\vec w$
\[
\vec w^T \nabla L_N(\vec z) \vec w \geq \frac{\kappa(\phi)}{N} \| \vec X \vec w \|_2^2\,,
\]
where $\vec X$ the above measurement matrix and $\vec z \in \Omega_{\phi}$. Hence, if we define $\vec \Delta = \wh{\vec z} - \vec z^{\star}$, we get (in order to form the definition of strong convexity) that
\[
L_N(\vec z^{\star} + \vec \Delta) - L_N(\vec z^{\star}) - \< \nabla L_N(\vec z^{\star}), \vec{ \Delta} \> \geq \frac{\kappa(\phi)}{N}\|\vec X \vec \Delta\|_2^2 = \kappa(\phi) \| \vec \Delta \|_L^2\,,
\]
since $\|\vec \Delta\|_L^2 = \vec \Delta^T \vec L \vec \Delta = \frac{1}{N} \vec \Delta^T \vec X^T \vec X \vec \Delta = \frac{1}{N}\| \vec X \vec \Delta\|_2^2$. Hence, the empirical negative log-likelihood objective is $\kappa(\phi)$-strongly convex at the true natural parameter vector $\vec z^{\star}$ with respect to the $\vec L$ semi-norm. 
\end{proof}


From the strong-convexity guarantee and, using \Cref{fact:Lbound}, we get that
\begin{equation}
\label{equation:error-bound}
\|\vec \Delta\|_L^2 \leq \frac{1}{\kappa(\phi)} \| \nabla L_N(\vec z^{\star}) \|_{L^{\dagger}}\,.
\end{equation}
\begin{claim}[Dual norm]
\label{claim:dualnorm}
Let $\vec L$ be the empirical Laplacian estimate matrix of $\vec Q$ and let $\vec L^{\dagger}$ the Moore-Penrose pseudo-inverse of $\vec L$. There exists a random vector $\vec Y \in \mathbb{R}^n$ such that 
\[
\| \nabla L_N(\vec z^{\star}) \|_{L^{\dagger}} 
= \frac{1}{N^2} (\vec X^T \vec Y)^T \vec L^{\dagger} (\vec X^T \vec Y)\,,
\]
where $\vec X$ is the $N \times n$ measurements matrix.
\end{claim}
\begin{proof}
We first provide a compact form the the gradient $\nabla L_N(\vec z^{\star})$. Again, following the analysis of~\cite{shah2016estimation}, we define a random vector $\vec Y \in \mathbb{R}^n$ with independent coordinates $(Y_1, \ldots, Y_n)$ where we set
$Y_i = F'(\<\vec z^{\star}, \vec v_i\>)/F(\<\vec z^{\star}, \vec v_i\>)$ with probability $F(\<\vec z^{\star}, \vec v_i\>)$ and we set $Y_i = -F'(\<\vec z^{\star}, \vec v_i\>)/(1-F(\<\vec z^{\star}, \vec v_i\>))$ otherwise. Recall that $F$, in our setting, is the Bradley-Terry function $F(x) = 1/(1+\exp(-x))$ and $\vec v_i \in \{-1,0,1\}^n$ is the $i$-th comparison vector indicating the drawn edge. Hence, after computing the gradient at the true parameter vector $\vec z^{\star}$, one gets
\[
\nabla L_N(\vec z^{\star}) = -\frac{1}{N} \vec X^T \vec Y\,.
\]
Using the above, \Cref{claim:dualnorm} follows. 
\end{proof}

We now introduce the matrix $\vec J := \frac{1}{\kappa(\phi) N^2} \vec X \vec L^{\dagger} \vec X^T$\,. Combining Equation \eqref{equation:error-bound} with \Cref{claim:dualnorm}, we get that
\[
\| \vec z^{\star} - \wh{\vec z}\|_L^2 \leq \vec Y^T \vec J \vec Y\,.
\]
Hence, in order to bound the $\vec L$ semi-norm estimation error, it suffices to control the quadratic form of the right hand side. Similarly to~\cite{shah2016estimation}, our goal is to apply the Hanson-Wright inequality (see~\Cref{lemma:hanson-wright}) in order to control the quadratic form.
\begin{lemma}
[\cite{hanson1971bound,  rudelson2013hanson}]
\label{lemma:hanson-wright}
Let $\vec Y \in \mathbb{R}^n$ be a random vector with independent zero-mean components, which are sub-Gaussian with parameter $K$ and let $\vec J \in \mathbb{R}^{n \times n}$ be an arbitrary matrix. Then there is a universal constant $c >0$ such that
\[
\Pr \left[ \left |\vec Y^T \vec J \vec Y - \E[\vec Y^T \vec J \vec Y] \right | > t \right] \leq 
2\exp \left( -c \min 
\left \{ 
\frac{t^2}{K^4 \| \vec J \|^2_{F}}, 
\frac{t}{K^2 \|\vec J \|_{\mathrm{2}}} 
\right\} \right) \,,
\]
for any $t > 0$.
\end{lemma}
Observe that $\E[\vec Y] = 0$ and, for any $i \in [n]$, we have that
\[
|Y_i| \leq \sup_{\vec z \in \Omega_{\phi}} \sup_{\vec x \in \mathbb{X}} \max \left \{ \frac{F'(\<\vec z, \vec x\>)}{F(\<\vec z, \vec x\>)}, \frac{F'(\<\vec z, \vec x\>)}{1-F(\<\vec z, \vec x\>)} \right \}\,,
\]
where $\mathbb{X} \subseteq \{-1,0,1\}^n$ is the space of all valid measurement vectors (they should induce edges in $\calE$). Hence, we get that 
\[
|Y_i| \leq \max_{x \in [-\log(\phi), \log(\phi)]} 
\max \Big\{ \frac{\exp(-x)/(1+\exp(-x))^2}{1/(1 + \exp(-x))}, \frac{\exp(-x)/(1+\exp(-x))^2}{\exp(-x)/(1 + \exp(-x))} \Big\} =: u(\phi)\,.
\]
Note that $u(\phi) \leq 1$, since, for $0 < b < 1$, it holds that $\max \left \{ \frac{1}{b}, \frac{1}{1-b} \right\} \leq \frac{1}{b(1-b)}$.
Since each $|Y_i|$ is bounded, we get that the random variables are $u(\phi)$-sub-Gaussian. Also, we have that
\[
\E[\vec Y^T \vec J \vec Y] \leq \E \Big [\|\vec Y\|_{\infty}^2 \mathrm{tr}(\vec J) \Big] \leq \frac{u(\phi)^2}{\poly(\kappa(\phi))} \frac{n}{N}\,.
\]
It remains to compute the two matrix norms. As in~\cite{shah2016estimation}, one gets that
\[
\| \vec J \|^2_F = \sum_{i,j} |\vec J_{ij}|^2 = \frac{n-1}{\poly(\kappa(\phi))N^2} = \Theta(n/N^2) \,,
\]
and 
\[
\| \vec J \|_{\mathrm{2}} = \max_{\vec x \neq \vec 0} \|\vec J \vec x \|_2/\|\vec x\|_2 = \frac{1}{\poly(\kappa(\phi))N} = \Theta(1/N)\,.
\]
Applying the Hanson-Wright concentration inequality, we conclude that
\[
\Pr \Big[ \| \vec z^{\star} - \wh{\vec z} \|_L^2 > c \cdot t \frac{u(\phi)^2}{\poly(\kappa(\phi))} \frac{n}{N} \Big] \leq e^{-t}\,,
\]
for all $t \geq 1$ and some universal constant $c > 0$. Recall that $\phi$ is the constant of \Cref{assumption:efficsample} and hence the term $\frac{u(\phi)^2}{\poly(\kappa(\phi))} \in \Theta(1)$. Finally, note that integrating the above tail bound, one gets the desired on expectation bound. 
\end{proof}

\subsection{Application of \Cref{algo:learn-shah} to Distributions}
\label{sec:distr-learn}
Given a target distribution $\D$, the associated natural parameters vector $\vec z = \log(\D)$ satisfies $\< \vec 1, \vec z \> \leq 0$. Hence, we have to shift it appropriately in order to apply \Cref{algo:learn-shah}. Recall that the condition
$\< \vec 1, \vec z\> = 0$ is only useful for identifiability purposes of the BTL model.

Given a target distribution $\D \in \Delta^n$, we convert it to the following weights vector: let $A = \sum_{x \in [n]} \log(\D(x))$ and set $\D'(x) = \exp(-A/n) \D(x)$.
So, we have to shift $\vec z$ in the direction $(-A/n) \cdot \vec 1$ and obtain the weight vector $ \vec z'$.
For the weights $(\D'(x))_{x \in [n]}$ with weight vector $\vec z'$, it holds that
\begin{enumerate}
    \item $\< \vec 1, \vec z'\> = \sum_{x \in [n]} (-A/n +\log(\D(x))) = 0$ and
    \item for any $x, y \in [n]$, it holds that $\D'(x)/\D'(y) = \D(x)/\D(y)$.
\end{enumerate}
Given an estimate for $\vec z'$, one can extract a good approximation for the natural parameter vector $\vec z$ of $\D.$ In particular, the following properties hold: first, given two distributions $\vec p \neq \vec q$, it holds that $\vec z'(\vec p) \neq \vec z'(\vec q)$, i.e., any $\vec z'$ is uniquely identified by its distribution since one of the following will hold: if $\sum_{i \in [n]} \log(p_i) = A = \sum_{i \in [n]} \log(q_i)$, then since there exists $j \in [n]$ so that $p_j \neq q_j$, we have that $z(\vec p)'_j = \log(p_j) - A/n \neq \log(q_j) -A/n = z(\vec q)'_j$. Otherwise, if, without loss of generality, $A_p = \sum_{i}\log(p_i) > \sum_{i} \log(q_i) = A_q$, we have that there exists $j \in [n]$ so that $p_j < q_j$ and so
$z(\vec p)'_j = \log(p_j) - A_p/n < \log(q_j) - A_q/n = z(\vec q)'_j$. Hence, any distribution $\D$ uniquely induces a vector $\vec z'$. Moreover, the actual natural parameter vector $\vec z$ can be estimated having an estimation for $\vec z'$. Assume that we have an estimate $\wh{\vec z'}$ that satisfies $\| \vec z' - \wh{\vec z'} \|_{2} \leq \epsilon$. 
We can estimate $\vec z$ using a vector $\wh{\vec z}$ (that can be extracted from our estimate $\wh{\vec z'}$) as follows:
Let $\vec z' = \vec z + C \vec 1$, where $C = -A/n.$
In order to initiate our exact sampling algorithm, an $L_{\infty}$ estimate for $\vec z$ is only required. We have that
\[
\| \vec z - \wh{\vec z} \|_{\infty}
= \| (\vec z' - C \vec 1) - (\wh{\vec z'} - C' \vec 1) \|_{\infty}
\leq \| \vec z' - \wh{\vec z'} \|_{\infty} + |C - C'|\,,
\]
where the estimate $\wh{\vec z'}$ satisfies $\| \vec z' - \wh{\vec z'} \|_{\infty} \leq \| \vec z' - \wh{\vec z'} \|_2 \leq \epsilon$ via \Cref{algo:learn-shah} and the shift constant can be estimated since we know that $\D$ is a distribution:
\[
\sum_{x \in [n]} \exp(\wh{z}_x) = 1 \iff \sum_{x \in [n]} e^{\wh{z'}_x - C'} = 1 \iff e^{C'} = \sum_{x \in [n]} e^{\wh{z'}_x}\,.
\]
Note that the right hand side of the above equation contains only our estimates. We have that $\sum_{x \in [n]} \exp(z_x) = 1$ and so we set $e^C :=  \sum_{x \in [n]} e^{z'_x}$. Since it holds that $|z'_x - \wh{z'}_x| \leq \epsilon$ for any $x \in [n]$, we have that
$|C - C'| \leq O(\epsilon)$. This gives that for the actual natural parameters it holds that $\| \vec z - \wh{\vec z} \|_{\infty} \leq O(\epsilon)$.

Hence, given a distribution $\D$, we can learn the weights $\vec z'$ in $L_2$ norm, satisfying the conditions of \Cref{algo:learn-shah}. From this estimation, we can get the estimation for the natural parameters $\vec z$ by removing the introduced shift. We remark that the above shifting methodology can be applied to any score vector $\vec z$ satisfying $\< \vec 1, \vec z \> = B \neq 0$.

\section{Local Sampling Schemes and Hypergraphs}
\label{appendix:hypergraphs}

In this section, we discuss how to extend our analysis to sets of size larger that 2. The hypergraph structure of Local Sampling Schemes can be settled as follows:
\begin{definition}[Hypergraph Structure of LSS]\label{definition:graph}
Let $\mathcal{Z}$ be a finite discrete domain and let $\Q$ be a distribution supported on subsets of $\mathcal{Z}$. Then, the hypergraph $G = (V,E)$ with vertex set $V = \mathcal{Z}$ and hyperedge set $E = \supp(\Q)$ is called a Local Sampling Scheme hypergraph. If $\Q$ is a pair distribution supported on $\mathcal{Z} \times \mathcal{Z}$, $G$ corresponds to a graph. 
\end{definition}
The general Markov Chain for distributions $\Q$, supported on sets $S \subseteq 2^{[n]}$ has the following transition probabilities:
\begin{equation}
    p_{xy} = \sum_{S \supseteq \{x,y\}} \Q(S) \frac{\D(y)}{\D(S)} = \D(y) \sum_{S \supseteq \{x,y\}} \frac{\Q(S)}{\D(S)}\,.
\end{equation}
For the transition $x \rightarrow y$, we can think of a flow $f_{xy}$ with mass $p_{xy}$ and, hence, induce a (simple) graph structure on a flow graph $F = (f_{xy})_{xy}$. Observe that any pair $(x,y)$ shares the same collection of hyperedges and, hence, we have that:
\begin{equation*}
    \frac{p_{xy}}{p_{yx}} = \frac{\D(y)}{\D(x)}\frac{\sum_{S \supseteq \{x,y\}} \Q(S)/\D(S)}{\sum_{S \supseteq \{x,y\}} \Q(S)/\D(S)} = \frac{\D(y)}{\D(x)}\,.
\end{equation*}
The ergodicity of the flow graph implies that the stationary distribution corresponds to $\D$ and is unique.
For what follows, we consider the case where any hyperedge has size $k$, i.e., $k$-uniform hypergraphs.
The Laplacian of the distribution $\Q$ can be generalized over hyperedges and have that
\[
\vec Q_{xy} = -\sum_{S \supseteq \{x,y\}}\Q(S)~\text{ for any } x,y \in [n]~ \text{ and } \vec Q_{xx} = (k-1)\sum_{S \ni x} \Q(S) ~\text{ for any } x \in [n]\,.
\]
In the following section, we describe some necessary notation.
\subsection{Notation}
We denote the set $\binom{[n]}{k}$ the family of size $k$ subsets of the ground set $[n]$.  Let $\Q$ be a distribution on $\binom{[n]}{k}$. Such a distribution will be called a $k$-set distribution.
Recall that a pair distribution is simply a $2$-set distribution. A sample from the Local Sampling Scheme $\mathrm{Samp}_k(\Q;\D)$, where $\Q$ is a $k$-set distribution, will be denoted
\[
(S, \vec v) \sim \mathrm{Samp}_k(\Q;\D)\,,
\]
where $S \in \supp(\Q)$ and, in the case where $i \in S$ won between the elements of $S$, we have that $\vec v = \vec e_i \in \mathbb{R}^n$ ($\vec v$ is the indicator vector of the winning node) and we have that $\Pr[\vec v = \vec e_i] = \D(i)/\D(S) \vec 1\{ i \in S \}$. 
For the hypergraph case, the canonical transition matrix induced by the Local Sampling Scheme will be denoted by $\vec P$. In the case $k=2$, this matrix was denoted by $\vec M$ (recall Equation \eqref{equation:trans-matrix}).
Let $\vec{P}$ denote the transition matrix of the Markov chain, associated with the Local Sampling Scheme $\mathrm{Samp}_k(\Q;\D)$, where $\Q$ corresponds to a $k$-set distribution. The entries of $\vec{P} = [\vec P_{xy}]_{x,y \in [n]}$ are defined as:
\begin{equation}
\label{equation:trans-matrix-hypergraph}
\vec P_{xy} = \sum_{S \supseteq \{x,y\}}\Q(S)\frac{\D(y)}{\D(S)}~\text{ for }~x \neq y~\text{ and }~\vec P_{xx} = 1 - \sum_{y \neq x} \vec P_{xy}~\text{ for } x \in [n]\,.
\end{equation}
Observe that the transition from $x$ to $y$ is performed when 
\begin{enumerate}
    \item a hyperedge $S \in \supp(\Q)$ is chosen (with probability $\Q(S)$) and both $x$ and $y$ lie in $S$,
    \item and the vertex $y \in [n]$ is the 'winning' node among the nodes of $S$, i.e., $y$ is drawn from the conditional distribution $\D_S$.
\end{enumerate}
There is a natural reduction to the graph case. Consider the marginals of $\Q$ to $2$-sets (i.e., edges). Then, one can define a Markov chain over a graph with $[n]$ nodes, whose transition probabilities are described by $\vec{P}$. The random walk has the following properties: Since any pair of vertices shares the same collection of hyperedges, we get that
\[
\frac{\vec P_{xy}}{\vec P_{yx}} = \frac{\D(y)}{\D(x)} \frac{\sum_{S \supseteq \{x,y\}}\Q(S) / \D(S)}{\sum_{S \supseteq \{x,y\}}\Q(S) / \D(S)} = \frac{\D(y)}{\D(x)}\,,
\]
for any pair of vertices. Moreover, the Markov chain is ergodic, since it is 
irreducible, since by the structure of $\Q$ and since $\D$ is supported on $[n]$, one eventually can get from every state to every other state with positive probability; and aperiodic, since it contains self-loops. Hence, it has a unique stationary distribution which coincides with $\D$, using the detailed balance equations.

\paragraph{Modifications of the LSS Conditions.} We can modify the information-theoretic connectivity condition (\Cref{assumption:identifiability})
for $k$-uniform hypergraphs and let $\calE$ the support of $\Q$.
Also, we have to consider the variation of \Cref{assumption:efficsample} over the support $\calE$ with
\[
1/\phi \leq \max_{(x,y) \in S \in \calE} \D(x)/\D(y) \leq \phi\,,
\]
for some constant $\phi$. We remark that the learning results of \cite{shah2016estimation} hold for $k = O(1)$. Similar, our learning tools which are modifications of the work of \cite{shah2016estimation} hold for the same regime. 

\subsection{Learning Phase for Hypergraphs}
\label{section:learn-hypergraph}
As in the learning phase of the $2$-set case, the learning algorithm for the $k$-set problem is essentially a variation of the analysis of~\cite{shah2016estimation}, but the steps are similar. Consider $N$ i.i.d. samples $(S_i, \vec v_i)$ drawn from the sampling oracle $\mathrm{Samp}_k(\Q; \D)$ with natural parameter vector $\vec z^{\star} \in \mathbb{R}^n$. 
The analysis that follows assumes that any $k$-set sample in the empirical likelihood lies in $\calE$, i.e., the support of $\Q$.
Our goal is to estimate the true parameter vector. We consider the negative empirical log-likelihood objective
\[
L_N(\vec z; \{ S_i, \vec v_i \}_{i \in [N]}) = -\frac{1}{N} \sum_{i=1}^N \< \vec z, \vec v_i \> - \log \sum_{j \in S_i} \exp(z_j)\,,
\]
and we optimize it over the parameter space
\[
\Omega_{\phi} = \{ \vec z \in \mathbb{R}^n : \<\vec 1, \vec z\> = 0,~ \max_{(x,y) \in S \in  \calE} | z_x - z_y| \leq \log(\phi)\}\,.
\]
\paragraph{Likelihood Objective and PL model.} We now observe that the above likelihood objective is directly connected to the Plackett-Luce model, which captures the process of choosing a single alternative from a given set, i.e., given a set $S$ of $m$ alternatives with values $w_1, \ldots, w_m$, the likelihood of choosing the $i$-th item is 
\[
F(w_i, w_1,..., w_{i-1}, w_{i+1},...,w_m) = \exp( w_i) / \sum_{j =1}^m \exp(w_j)\,.
\]
Note that the negative log-likelihood of the Plackett-Luce model is exactly the same as the single sample version of our objective function.
Observe that this function is shift-invariant and its value is independent of the ordering of the last $(m-1)$ elements. Shift invariance is crucial: if one does not work in the subspace $\{ \vec z : \<\vec 1, \vec z\> = 0\}$, then neither our problem nor the problem of determining the values of the alternatives in the Plackett-Luce model are identifiable, since any solution of the form $\vec z^{\star} + c \vec 1$ is valid. Technically, shift invariance implies that $\vec 1$ lies in the nullspace of the Hessian of the negative log-likelihood (and this is where the first constraint of the set $\Omega_{\phi}$ arises). 

\paragraph{Hessian matrix.} Let us introduce the vector $\vec e^z(S) = (\exp(z_i))_{i \in S} \in \mathbb{R}^k$ for an arbitrary set $S \subseteq \binom{[n]}{k}$. After standard computations (see also~\cite{shah2016estimation}), we get that
\[
\nabla^2_{\vec z} \Big (-\< \vec z, \vec v \> + \log \sum_{j \in S} \exp(z_j) \Big)
= 
\frac{ \Big ( \sum_{j \in S} \exp(z_j) \Big) \mathrm{diag}(\vec e^z(S)) - \vec e^z(S)\vec e^z(S)^T}{ \Big ( \sum_{j \in S} \exp(z_j) \Big )^2}\,,
\]
and hence
\[
\nabla^2_{\vec z} L_N(\vec z; \{S_i, \vec v_i\}_{i \in [N]}) = 
\frac{1}{N} \sum_{i=1}^N \frac{\<\vec e^z(S), \vec 1\>\mathrm{diag}(\vec e^z(S_i)) - \vec e^z(S_i) \vec e^z(S_i)^T}{ \Big ( \< \vec e^z(S_i), \vec 1\> \Big )^2}\,.
\]
\paragraph{Quantitative Strong Convexity of PL model.} Consider an arbitrary direction $\vec v \in \mathbb{R}^k$. Without loss of generality, assume that $S = [k]$. Then, it holds that
\[
\vec v^T \Big ( \vec e^z(S)\vec e^z(S)^T \Big ) \vec v = \sum_{i,j = 1}^k v_i v_j \exp(z_i + z_j)  \leq \sum_{i=1}^k \exp(z_i) \sum_{j=1}^k v_j^2 \exp(z_j) = \< \vec e^z(S), \vec 1\> \mathrm{diag}(\vec e^z(S))\,,
\]
using the Cauchy-Schwarz inequality, where equality holds if and only if $\vec v \in \mathrm{span}(\vec 1)$. Hence, for the Plackett-Luce function $F : \mathbb{R}^k \rightarrow \mathbb{R}$, we have that
\[
\lambda_2(-\log(F(\vec w))) > 0,~\text{ for any }~\vec w \in \Omega_{\phi}\,.
\]
In order to quantify the strong convexity parameter, we will make use of the second constraint. For the Plackett-Luce function $F$, we have that
\[
\nabla^2 (-\log(F(\vec w)) \succeq \vec H\,,
\]
for some $k \times k$ symmetric matrix $\vec H$ with $\lambda_2(\vec H) > 0.$ We can observe that the matrix
\[
\vec H = \beta(\phi) (\vec I - \vec 1 \vec 1^T)\,,
\]
where $\beta(\phi) = \min_{\vec z \in \Omega_{\phi}} \lambda_2 \left( \frac{\<\vec e^z(S), \vec 1\>\mathrm{diag}(\vec e^z(S)) - \vec e^z(S) \vec e^z(S)^T}{ \Big ( \< \vec e^z(S), \vec 1\> \Big )^2}\right)$, satisfies the strong log-concavity condition for the Plackett-Luce model, i.e., the function $F : \mathbb{R}^k \rightarrow \mathbb{R}$.

\paragraph{Underlying Laplacian matrix and estimation.} Similarly to the case $k=2$, our goal is to establish strong convexity for the empirical negative log-likelihood around the true parameters $\vec z^{\star} \in \Omega_{\phi}$ with respect to the $\vec L$ semi-norm. Hence, we have to first introduce the appropriate Laplacian matrix $\vec L$. In the $k$-set setting, we have that
\[
\vec Q_{x,y} = -\sum_{S \ni x,y} \Q(S) \text{ for any}~x,y \in [n],~\vec Q_{x,x} = (k-1) \sum_{S \ni x} \Q(S) \text{ for any}~x \in [n]\,. 
\]
Observe that $\vec Q \vec 1 = 0$ (since each set $S \subseteq \binom{[n]}{k}$ that contains $x$, also contains other $k-1$ elements). Recall that for the case $k=2$, the estimate for the Laplacian matrix $\vec Q$ is given by the empirical matrix $\vec L = \frac{1}{N} \sum_{i = 1}^{N} \vec v_i \vec v_i^T = \frac{1}{N} \vec X^T \vec X$ (see~\Cref{appendix:learningalgo}). In the case $k>2$ case, the Laplacian estimate for the matrix $\vec Q$ is given by the $n \times n$ matrix
\[
\vec L = \frac{1}{N} \sum_{i=1}^N \vec E_i (k\vec I - \vec 1 \vec 1^T) \vec E_i^T\,,
\]
where the vectors $\vec v_i$ are substituted by the $(n \times k) $ matrices $\vec E_i$, where each one of the $k$ columns of $\vec E_i$ is a unit vector. The $k$ non-zero elements of $\vec E_i$ correspond to the $k$ elements that lie in the $i$-th drawn set $S_i$, i.e., $\vec E_i = \Big [ \vec e_{i_1} \Big | \vec e_{i_2} \Big | \ldots \Big | \vec e_{i_k} \Big ]$ for $S_i = \{ i_1, i_2, \ldots, i_k\}.$ 

In the estimation part\footnote{Recall that the learning algorithm does not require the knowledge of the Laplacian matrix $\vec Q$. The concentration result for the empirical estimation of the matrix $\vec Q$ using the matrix $\vec L$ enables us to express our sample complexity results using the second smallest eigenvalue of the true (population) matrix $\vec Q$ and not its empirical estimate.}, our goal is to estimate the unknown matrix $\vec Q$.
Using concentration results of the Fiedler eigenvalue of sums of Hermitian matrices \cite{tropp2015introduction}, using $O(\log(n)/\lambda(\vec Q))$ samples, one can compute an estimate $\vec L$ (that corresponds to the above empirical estimate) that approximates the true $\lambda(\vec Q)$ with $(1+\epsilon)$ multiplicative error.

We have that
\[
\vec L = \frac{1}{N} \sum_{i =1}^N \vec X_i\,,
\]
where\footnote{For instance, let $n = 5, k=3$ and $S = \{1,2,3\}$.\\ We have that $\vec E = [\vec e_1 | \vec e_2 | \vec e_3] = \begin{bmatrix}
1 & 0 & 0\\
0 & 1 & 0\\
0 & 0 & 1\\
0 & 0 & 0\\
0 & 0 & 0
\end{bmatrix}$ and $\vec X = \begin{bmatrix}
2 & -1 & -1 & 0 & 0\\
-1 & 2 & -1 & 0 & 0\\
-1 & -1 & 2 & 0 & 0\\
0 &  0 & 0  & 0 & 0\\
0 &  0 & 0  & 0 & 0
\end{bmatrix}$} $\vec X_i = \vec E_i (k\vec I - \vec 1 \vec 1^T) \vec E_i^T$, i.e., the estimate $\vec L$ is a sum of independent Hermitian matrices with $\lambda_{max}(\vec X_i) = \Theta(k).$ Using \Cref{prop:conc} in a similar manner as in \Cref{lemma:concetration-fiedler}, we get that
\[
\Pr_{\vec X_{1..m} \sim \Q^{\otimes m}} \left[ \lambda(\vec L) \leq (1-\epsilon)  \lambda(\vec Q) \right]
\leq 
n \exp \left( -\epsilon^2 m \frac{\lambda(\vec Q)}{2k} \right)\,.
\]
Hence, it suffices to draw at least
\[
m = O \left( \frac{k}{\epsilon^2 \cdot \lambda(\vec Q)} \log(n/ \delta) \right)
\]
in order to guarantee the desired concentration bound with confidence at least $1-\delta$.

\paragraph{Strong Convexity of Empirical NLL w.r.t. $\|\cdot \|_L$.} Having introduced the empirical Laplacian matrix, we are able to get the desired strong convexity result for the empirical negative log-likelihood with respect to the $\vec L$ semi-norm. Following the analysis of~\cite{shah2016estimation}, we get that, for any vector $\vec w \in \mathbb{R}^d$ and parameter $\vec z \in \Omega_{\phi}$,
\[
\vec w^T \nabla^2 L_N(\vec z) \vec w \geq \frac{\lambda_2(\vec H)}{k} \frac{1}{N} \sum_{i =1}^N \sum_{j=1}^{k} \vec 1\{ \vec v_i = \vec e_j \} \vec w^T \vec E_i (k \vec I - \vec 1 \vec 1^T) \vec E_i^T \vec w\,,
\]
where $\vec H = \beta(\phi)(\vec I - \vec 1\vec 1^T)$.  Hence, we get that
\[
\vec w^T \nabla^2 L_N(\vec z) \vec w \geq \frac{\lambda_2(\beta(\phi)(\vec I - \vec 1\vec 1^T))}{k} \| \vec w \|_L^2\,.
\]
Note that $\lambda_2(\vec I - \vec 1 \vec 1^T) = 1$.
\footnote{The characteristic polynomial of $\vec A = \vec I - \vec 1 \vec 1^T$ is
$\det(\vec A - \lambda \vec I) = (\lambda-1)^{n-1} (\lambda + n-1).$
} Consequently, the empirical negative log-likelihood is $\beta(\phi)/k$-strongly convex around the true parameters $\vec z^{\star} \in \Omega_{\phi}.$
Using \Cref{fact:Lbound}, we get that
\[
\| \vec z^{\star} - \wh{\vec z} \|_{L}^2 \leq \frac{k^2}{\beta(\phi)^2} \| \nabla L_N(\vec z^{\star}) \|_{L^{\dagger}}^2\,.
\]
\paragraph{Bounding the (expected) dual norm $\| \cdot \|_{L^{\dagger}}$.} Following the exact analysis of~\cite{shah2016estimation} for the expected value of the dual norm, we get that for our setting
\[
\E \Big[ \< \nabla L_N(\vec z^{\star}), \vec L^{\dagger} \nabla L_N(\vec z^{\star}) \> \Big] \leq \frac{n}{N} \sup_{\vec z \in \mathbb{R}^k : |z_x - z_y| \leq \log(\phi)} \| \nabla \log F(\vec z) \|_2^2\,,
\]
where $F$ is the Plackett-Luce function and assuming that $S = [k]$ (and $S$ satisfies the modified mass condition, i.e., $S$ lies in the support of $\Q$). We have that
\[
\nabla_{\vec z} \log F(\vec z) = \vec e_i - \frac{(\exp(z_1),..., \exp(z_k))^T}{\sum_{j \in [k]} \exp(\vec z_j)}\,,
\]
for some $i \in [k]$. Hence, we get that
\[
\| \nabla_{\vec z} \log F(\vec z) \|_2^2 = 
\left (1 - \frac{e^{z_i}}{\sum_{j \in [k]} e^{z_i}} \right )^2 
+
\frac{\sum_{j \in [k] \setminus \{i\}} e^{2z_j} }{ (\sum_{j \in [k]} e^{z_j})^2 }
=
\frac{ (\sum_{j \in [k] \setminus \{i\}} \exp(z_j))^2 + \sum_{j \in [k] \setminus \{i\}} \exp(2z_j)  }{(\sum_{j \in [k]} \exp(z_j))^2}
\]
Over the optimization space $\Omega_{\phi}$ and since the natural parameters $\vec z_x = \log(\D(x)) \leq 0$, we have that
\[
\frac{1}{k \phi} \leq \frac{\D(x)}{k \max_{j \in S} \D(j)} \leq \frac{\D(x)}{\sum_{j \in S = [k]} \D(j)} \leq \frac{\D(x)}{k \min_{j \in S} \D(j)} \leq \phi/k
\]
Hence, we have that
\[
\left ( \frac{\sum_{j \in [k] \setminus \{i\}} \D(j)}{\sum_{j \in [k]} \D(j) } \right )^2 \leq ((1-1/k) \phi)^2\,,
\]
and
\[
\frac{ \sum_{j \in [k] \setminus \{i\}} \D(j)^2 }{ (\sum_{j \in [k] } \D(j))^2} = \sum_{j \in [k] \setminus \{i\}} \frac{\D(j)}{\sum_{j' \in [k] } \D(j')}
\frac{\D(j)}{\sum_{j' \in [k] } \D(j')}
\leq (k-1) \frac{\phi^2}{k^2}\,.
\]
Finally, we get that
\[
\| \nabla_{\vec z} \log F(\vec z) \|_2^2 \leq \phi^2 \frac{k-1}{k}\,.
\]
\paragraph{Conclusion of the Learning Phase in $k$-sets.} We get that
we can estimate the true parameters in $\vec L$ semi-norm, i.e.,
\[
\| \vec z^{\star} - \wh{\vec z} \|_L^2 \leq \frac{k^2}{\beta(\phi)^2} \frac{n}{N} \phi^2 \frac{k-1}{k}.
\]
Hence, we can transfer this result to an $L_2$ bound and get
\[
\| \vec z^{\star} - \wh{\vec z} \|_2^2 \leq \frac{k^2}{\beta(\phi)^2} \frac{n}{N \lambda_2(\vec L)} \phi^2 \frac{k-1}{k}.
\]
Using $N = O \Big(\frac{n k^2}{\lambda_2(\vec L) \beta(\phi)^2} \cdot \frac{1}{\epsilon^2} \Big) =_{\epsilon} O \Big( \frac{n k^2}{\lambda(\vec Q) \beta(\phi)^2} \cdot \frac{1}{\epsilon^2} \Big)$ samples from the oracle $\mathrm{Samp}_k(\Q; \D)$ on the support set $\calE$, we can learn the parameter vector in $L_2$. At this point, the proof of the learning phase is similar to the $k=2$ case: We can apply the trick of \Cref{sec:distr-learn}
and get an $L_{\infty}$ bound. Hence, using \Cref{lemma:stability}, we can estimate the target distribution with relative error. 

\subsection{Downscaling Phase for Hypergraphs}
Let $\wh{\vec z}$ be the estimate of the natural parameters vector and let $\wt{\D}$ the correspond distribution. We use the estimate $\wt{\D}$ of the target distribution $\D$ to transform the Markov chain of the $k$-set Local Sampling Scheme into another almost uniform stationary distribution, as in the case $k=2$. Using the Bernoulli downscaling mechanism, for the pair $(x,y)$, we have that
\[ 
\wt{p}_{xy} = \D(y) \frac{\wt{\D}(x)}{\wt{\D}(y)} \sum_{S \ni x,y}\frac{\Q(S)}{\D(S)}  \approx \D(x) \sum_{S \ni x,y}\frac{\Q(S)}{\D(S)} = p_{yx}\,.
\]
We have that
\[
\frac{\wt{p}_{xy}}{\wt{p}_{yx}} = \frac{\D(y) \frac{\wt{\D}(x)}{\wt{\D}(y)} \sum_{S \ni x,y}\frac{\Q(S)}{\D(S)} }{\D(x) \sum_{S \ni x,y}\frac{\Q(S)}{\D(S)}} = \frac{\D(y)/\wt{\D}(y)}{\D(x)/\wt{\D}(x)}\,.
\]
The modified transition matrix (as in the $k=2$ case) has an almost uniform stationary distribution. Also, it has an absolute spectral gap of same order as the matrix $\wt{\vec P}$ that can be expressed as
\[
\wt{\vec P}_{xy} =  \Big( \frac{1}{2} + \epsilon_{xy} \Big) \sum_{S \ni x,y} \Q(S) \,, ~~ \wt{\vec P}_{yx} = \Big( \frac{1}{2} + \epsilon_{yx} \Big) \sum_{S \ni x,y} \Q(S) \,, 
\]
and $\wt{\vec P}_{xx} = 1 - \sum_{y \neq x} \wt{\vec P}_{xy}$. Hence, we get that
\[
\wt{\vec P} = \vec I - \frac{1}{2}\vec Q + \vec Q \circ [\epsilon_{xy}]\,.
\]
The expression of the modified transition matrix $\wt{\vec P}$ is exactly similar to the $k=2$ case (see Equation \eqref{equation:balanced}). Spectral analysis as in \Cref{lemma:tranform-properties} of this matrix will result to the following properties:
\begin{itemize}
    \item[$(i)$] The transition matrix $\wt{\vec{P}}$ has stationary distribution $\wt{\pi}_0(x) = \Theta(1/n)$ for all $x \in [n]$, i.e., an \emph{almost} uniform probability measure.
    \item[$(ii)$] The absolute spectral gap $\Gamma(\wt{\vec{P}})$ of $\wt{\vec{P}}$ and the minimum non-zero eigenvalue $\lambda(\vec{Q})$ of the Laplacian matrix $\vec{Q}$ satisfy $\Gamma(\wt{\vec{P}}) = \Omega(\lambda(\vec{Q}))$. 
    \item[$(iii)$] The mixing time of the transition matrix $\wt{\vec{P}}$ is $T_{\mathrm{mix}}( \wt{\vec{P}}; 1/4) = O(\log(n)/\lambda(\vec{Q})),$ where $\lambda(\vec{Q})$ is the minimum non-zero eigenvalue of the Laplacian matrix $\vec{Q}$.
\end{itemize}


\subsection{CFTP Phase for Hypergraphs}
The structure of the $k$-set algorithm remains similar to the $k=2$ case, with the exception that the drawn samples are now hyperedges. We continue with the modified \Cref{algo:exact-hypergraph}. Recall that, for each drawn sample $(S,\vec v) \sim \mathrm{Samp}_k(\Q; \D)$, the vector $\vec v \in \{\vec e_1, ..., \vec e_n\}$ is the indicator vector of the winning node of the set $S \subseteq \binom{[n]}{k}$.

\begin{algorithm}[ht] 
\caption{Exact Sampling from $k$-set Local Sampling Schemes}
\label{algo:exact-hypergraph}
\begin{algorithmic}[1]
\Procedure{ExactSampler-$k$-Set}{$\vec p$} \Comment{\emph{Sample access to the LSS oracle $\mathrm{Samp}_k(\Q; \D)$.}}
\State $t \gets 0$
\State $F_0(x) \gets x$, for any $x \in [n]$
\State $\textbf{while}~F_{t}~\text{has not coalesced}~\textbf{do}$ \Comment{\emph{While no coalescence has occured.}}
\State $~~~~~~~~ t \gets t - 1$ 
\State~~~~~~~~Draw $(S, \vec v)$ where $S \in \calE$ and $\vec v = \vec e_i$ for some $i \in [n]$
\State~~~~~~~~$y \gets $ position where $\vec v[y] = 1$ \Comment{\emph{Find the winning node $y \in S$.}}

\State ~~~~~~~~\textbf{for}~$x = 1\ldots n$~\textbf{do} \Comment{\emph{In order to update state $x$.}}


\State ~~~~~~~~~~~~~~~~ \textbf{if} $x \notin S ~\mathrm{or}~ x= y$  \textbf{then} $F_{t}(x) \gets F_{t+1}(x)$
\State ~~~~~~~~~~~~~~~~ \textbf{else}
\[
F_{t}(x) \gets \left\{\begin{array}{lr}
        \color{black} F_{t+1}(w), & \color{black} \text{with probability } \min \{ p(x) / p(w),1 \} \color{black} \\
        \color{black} F_{t+1}(x), & \color{black} \text{otherwise } \color{black}
        \end{array}\right.
  \]
\State ~~~~~~~~\textbf{end} 
\State $\textbf{end}$
\State Draw $C \sim \mathrm{Be}(p(F_{t}(1)))$ \Comment{\emph{$\mathrm{Be}(p)$ is a $p$-biased Bernoulli coin.}}
\State \textbf{if} $C =1$ \textbf{then} Output $F_{t}(1)$ \textbf{else} Output $\perp$
\Comment{\emph{Output the perfect sample or reject.}}
\EndProcedure

\item[]

\Procedure{ExactSampler-$k$-Set-WithLearning}{$\delta$} \Comment{\emph{The algorithm of \Cref{cor:sampling-hypergraph}}.}
\State $\wt{\D} \gets \textsc{Learn}(\epsilon := 1/\sqrt{n}, \delta)$ \Comment{\emph{Learn $\D$ in relative error as in \Cref{section:learn-hypergraph}}.}
\State x $ \gets \perp$
\State \textbf{while} x $= \perp$ \textbf{repeat}
\State ~~~~~~~~ x $\gets$ \textsc{ExactSampler-$k$-Set}$(\wt{\D})$ 
\State Output x
\EndProcedure

\end{algorithmic}
\end{algorithm}

\subsection{Conclusion}
Let $N_{\mathrm{Learn}}$ and $N_{\mathrm{CFTP}}$ be the expected number of samples required for a single execution of the learning algorithm for $k$-sets and of the parameterized CFTP algorithm respectively (see \Cref{algo:exact-hypergraph}).
Under \Cref{assumption:identifiability} (for the distribution $\Q$ over hyperedges) and the modified version of \Cref{assumption:efficsample}, the expected sample complexity of the generalized algorithm for the $k$-set case is
\[
N_{\mathrm{Learn}}(\epsilon := 1/\sqrt{n}) + \Theta(n)  \cdot N_{\mathrm{CFTP}} = 
O \left(\frac{n^2 k^2}{\lambda(\vec Q) \beta(\phi)^2}  \right) 
\,.
\]
Note that $\beta(\phi)$ is constant under \Cref{assumption:efficsample}.
As explained in~\cite{shah2016estimation}, it is natural to consider the regime where $k = O(1)$. In this regime, the learning bounds match the CFTP dependence on $n$ under $\wt{O}$ notation. 

\begin{corollary}[Exact Sampling from $k$-Set LSS]\label{cor:sampling-hypergraph}
For any positive constant integer $k > 2$, 
under \Cref{assumption:identifiability} (for the distribution $\Q$ over $k$-sets) 
and \Cref{assumption:efficsample},
there exists an algorithm (\Cref{algo:exact-hypergraph}) that draws an expected number of $\wt{O} 
\left( \frac{n^2 k^2 }{\lambda(\vec Q)} \right)$
samples from a $k$-Set Local Sampling Scheme $\mathrm{Samp}_k(\Q;\D)$, and
generates a sample distributed as in $\D$.
\end{corollary}

\newpage

\appendix

\newpage

\end{document}